\documentclass{article}

\usepackage[square,numbers]{natbib}

\usepackage{multibib}
\newcites{appendix}{Appendix References}

\usepackage[verbose=true,letterpaper]{geometry}

\usepackage[utf8]{inputenc} 
\usepackage[T1]{fontenc}    
\usepackage{hyperref}       
\usepackage{url}            
\usepackage{booktabs}       
\usepackage{amsfonts}       
\usepackage{nicefrac}       
\usepackage{microtype}      
\usepackage[table]{xcolor}         

\usepackage{amsmath,amsfonts,bm}

\def\eqref#1{equation~\ref{#1}}

\def\1{\bm{1}}

\def\vs{{\bm{s}}}

\DeclareMathAlphabet{\mathsfit}{\encodingdefault}{\sfdefault}{m}{sl}
\SetMathAlphabet{\mathsfit}{bold}{\encodingdefault}{\sfdefault}{bx}{n}

\usepackage{graphicx}
\usepackage{cleveref}
\usepackage{booktabs, multirow} 
\usepackage{soul}
\usepackage{changepage,threeparttable} 
\usepackage{siunitx}

\usepackage{amssymb,amsmath,amsthm}
\usepackage{mathtools} 

\usepackage{subcaption}

\usepackage{listings}

\usepackage{eso-pic}
\usepackage{xspace}
\makeatletter
\DeclareRobustCommand\onedot{\futurelet\@let@token\@onedot}
\def\@onedot{\ifx\@let@token.\else.\null\fi\xspace}

\def\eg{e.g\onedot} 
\def\ie{i.e\onedot}

 \def\vs{vs\onedot}

\makeatother

\newcommand{\linewidthbox}[1]{%
  \begingroup
    \sbox0{\ignorespaces#1\unskip}%
    \leavevmode
    \ifdim\wd0>\linewidth
      \hbox to\linewidth{%
        \hss\resizebox{\linewidth}{!}{\copy0 }\hss
      }%
    \else
      \copy0 %
    \fi
  \endgroup
}

\crefname{section}{Sec.}{Sec.}
\crefname{appendix}{App.}{App.}
\crefname{definition}{Def.}{Defs.}
\crefname{proposition}{Prop.}{Props.}

\newtheorem{theorem}{Theorem}

\newtheorem{lemma}{Lemma}

\graphicspath{ {./figures/} }

\newcommand{\methodname}{Distributional Mixture of Experts}
\newcommand{\methodacronym}{DMoE}

\title{Distribution Learning for Molecular Regression}

\author{Nima Shoghi\textsuperscript{$\ast$1} \; Pooya Shoghi \; Anuroop Sriram\textsuperscript{2} \; Abhishek Das\textsuperscript{$\ast$}\\
   \\ {\textsuperscript{1}{\small{Georgia Institute of Technology}}}
   \\ {\textsuperscript{2}{\small{Fundamental AI Research (FAIR) at Meta}}}
   \\ {\textsuperscript{$\ast$}{\small{Work done while at FAIR}}}
   \\ Correspondence to: \texttt{nimash@gatech.edu}
}

\begin{document}
\newgeometry{
      textheight=9in,
      textwidth=5.5in,
      top=1in,
      headheight=12pt,
      headsep=25pt,
      footskip=30pt
}

\maketitle

\begin{abstract}
      Using ``soft'' targets to improve model performance has been shown to be effective in classification settings, but the usage of soft targets for regression is a much less studied topic in machine learning. The existing literature on the usage of soft targets for regression fails to properly assess the method's limitations, and empirical evaluation is quite limited. In this work, we assess the strengths and drawbacks of existing methods when applied to molecular property regression tasks. Our assessment outlines key biases present in existing methods and proposes methods to address them, evaluated through careful ablation studies. We leverage these insights to propose Distributional Mixture of Experts (DMoE): A model-independent, and data-independent method for regression which trains a model to predict probability distributions of its targets. Our proposed loss function combines the cross entropy between predicted and target distributions and the L1 distance between their expected values to produce a loss function that is robust to the outlined biases. We evaluate the performance of DMoE on different molecular property prediction datasets -- Open Catalyst (OC20), MD17, and QM9 -- across different backbone model architectures --  SchNet, GemNet, and Graphormer. Our results demonstrate that the proposed method is a promising alternative to classical regression for molecular property prediction tasks, showing improvements over baselines on all datasets and architectures.
\end{abstract}

\section{Introduction}
Graph Neural Networks (GNNs) have been shown to be quite successful at molecular property
prediction~\cite{duvenaud_convolutional_2015,gilmer_neural_2017,schutt2017schnet,klicpera2020directional}. In these methods, systems are typically represented as graphs with atoms as nodes,
and properties of interest as graph-level (\eg energy) or node-level (\eg atomic force) targets.
%

Orthogonal to the underlying GNN architecture, there is an abundance of methods
that have been proposed to improve model generalization.
Data augmentation methods -- ~\eg by adding Gaussian noise to atomic positions~\cite{godwin2021very} --
modify the input dataset to increase diversity and make the GNN robust to transformations
without additional annotated data.
Jointly training the GNN on auxiliary tasks (in addition to the main task of interest)
-- \eg, by adding an auxiliary position denoising loss~\cite{godwin2021very} --
has been shown to lead to better representations.
%
Regularization techniques -- ~\eg, penalizing overconfident predictions~\cite{Pereyra2017RegularizingNN}, randomly dropping nodes~\cite{do_dropnode} and edges~\cite{rongdropedge} in the input graph, and normalization layers~\cite{caigraphnorm,zhaopairnorm,limsgnorm,zhaogroupnorm}
-- add further constraints to improve generalization.

We focus on a parallel set of techniques that change the target representations.
In classification tasks, label smoothing~\cite{Szegedy_2016_CVPR} is one
such technique that modifies the target distribution to be a mixture of categorical
(\ie one-hot) and uniform.
In regression tasks -- which are more common for molecular prediction -- \citet{imani2018improving}
proposed histogram regression,
where the target scalars are converted to `soft' probability
distributions, and the model is trained to predict these distributions.

In this paper, we begin by conducting a thorough analysis of \citet{imani2018improving}'s
histogram regression and find that it is not
well-suited for molecular property prediction tasks.
We identify key biases in histogram regression and propose strategies to address them.
Finally, we combine these strategies to propose a model-independent and dataset-independent
technique that improves performance on a host of molecular property prediction
datasets -- Open Catalyst~\cite{ocp_dataset}, MD17~\cite{chmiela2017machine}, QM9~\cite{doi:10.1021/ci300415d, ramakrishnan2014quantum} --
across different backbone GNN architectures -- SchNet~\cite{schutt2017schnet},
GemNet~\cite{klicpera2021gemnet}, Graphormer~\cite{shi2022benchmarking}.

On the OC20 IS2RE dataset, our method shows an average of 8.4\% and 34.3\% relative improvement on energy MAE and energy within threshold metrics, respectively, across the validation ID and OOD adsorbate splits and an average of 4.6\% and 18.4\% relative improvement over all validation splits. On the QM9 dataset, our method shows an average improvement of 5\% in threshold accuracy. Finally, on the MD17 dataset, our method shows an average improvement of 0.17\% on energy MAEs.

We utilize our technique to create a variant of GemNet~\cite{klicpera2021gemnet} that achieves competitive results to the current state of the art on the OC20 dataset's direct IS2RE prediction task.
%


\section{Histogram Regression}
\label{sec:historeg}
\begin{figure*}[h]
      \includegraphics[width=1.0\textwidth]{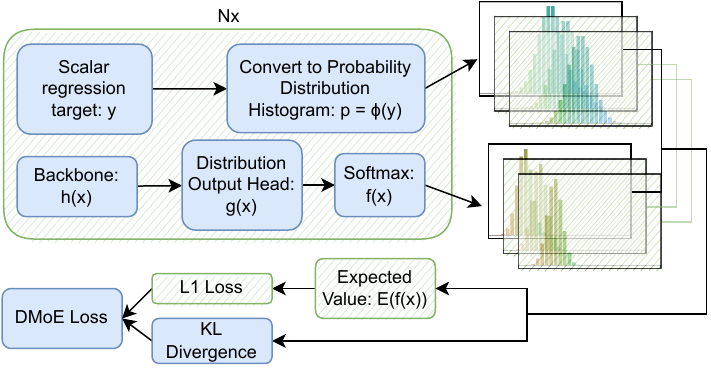}
      \caption{Overview of histogram regression. DMoE additions are highlighted in green.}
      \label{fig:concept:histogram_regression}
\end{figure*}

For bounded regression problems, histogram regression~\cite{imani2018improving}, illustrated in \cref{fig:concept:histogram_regression}, is a simple and effective technique for increasing performance.
Assume we have a scalar target \(y \in [y_{min}, y_{max}]\) and a backbone model \(h(x)\). In histogram regression, instead of training a model to directly predict \(y\), we train a model to predict a target probability distribution that we construct from \(y\). Traditionally, this constructed target distribution is a Gaussian distribution with a mean of \(y\) and with a fixed variance, \(\sigma^2\), set as a hyperparameter.
This target distribution is then discretized to a histogram representation with a fixed number of bins, \(N\), where each bin has a fixed width, \(w\), and the bin boundaries are set as \(\vec{b}_i = i \times w + y_{min}\) for \(i = 1, 2, \ldots, N\). The mass at each bin, \(p_i\), is equal to \(P(\vec{b}_i + w) - P(\vec{b}_i)\), where \(P(x)\) is the CDF of the Gaussian distribution.

We then attach a histogram output head, \(g(x)\) to the backbone model and take the softmax of the histogram output head, giving us the final model output, \(\hat{Y} = f(x)\).
\begin{align}
      f(x)               & = \text{softmax}(g(x)) \\
      \text{where } g(x) & = \text{MLP}(h(x))
\end{align}

The loss function -- called the histogram loss -- is then the cross entropy between the model's predicted distribution and the induced distribution.
\begin{equation}
      L_{HL}(\hat{Y}, Y) = -\sum_i{Y_i \log(\hat{Y}_i)}
\end{equation}

\citet{imani2018improving} provide theoretical justification for the
improved performance of histogram regression over traditional regression.
By comparing the norm of the loss gradient for histogram loss~\vs a mean
squared error loss, they show that histogram losses produce smoother, more
stable gradients.
We refer the reader to Sec. 3 in~\citet{imani2018improving} for a detailed exposition on this.

However, existing theoretical results are insufficient at explaining the
observed strengths and drawbacks of histogram regression.
%
%
Through careful empirical evaluation, we identify two primary biases that
are inherent in histogram regression -- 1) Distribution Quantization Error and 2) Histogram Distance Bias. Additional biases are explored in \cref{appendix:biases}.

\subsection{Distribution Quantization Error}
\label{sec:distribution-quantization-error}

\begin{figure*}[h]
      \includegraphics[width=1.0\textwidth]{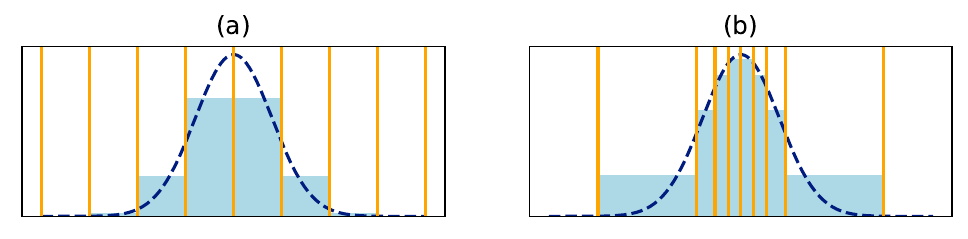}
      \caption{For a normally distributed regression target, notice how traditional uniform histograms, shown in (a), yields a much lower precision (and thus a higher error) than the histogram with normally distributed bins, shown in (b). Using \cref{eq:quantization-error}, we compute the quantization error of (a) to be \(6.07\) and (b) to be \(2.53\).}
      \label{fig:method:normal-bin-dist}
\end{figure*}

Distribution quantization error refers to the unavoidable error that comes from utilizing a discrete
distributional representation for regression targets.
In other words, this is quantization error of going from the induced target
probability histogram to a scalar-valued target.

We empirically demonstrate this error by utilizing a threshold accuracy metric.
A threshold accuracy metric is one that measures the percentage of predictions
that are within a certain threshold of the true target value, rather than the
absolute value of the difference between the prediction and the true target value.
Intuitively, quantization error or error due to coarse resolution will contribute
towards an absolute error metric but not towards binary threshold metrics (as long
as predictions are within the threshold).

To address this error:

\begin{itemize}
      \item We allow the probability distribution histogram to be non-linear. In other words, we allow the histogram bins to be non-equidistant and thus have different widths. This allows us to increase the resolution of our histograms (by using smaller bins) in intervals with higher densities of values (see \cref{fig:method:normal-bin-dist}). \Cref{hist:qerror:bin-dist} describes, in more detail, how this is achieved.
      \item Instead of a single output head, we use a series of output heads,
            where each head's target histogram uses slightly shifted bin
            endpoints. Note that simply using multiple output heads -- all with the
            same target histogram (\ie the same bin endpoints) -- is still susceptible
            to quantization error, so using shifted bin endpoints across output
            heads is essential. The multi-histogram loss value is then the (optionally weighted) mean of the loss for each individual histogram. The optimized implementation of this procedure is described in \cref{appendix:optimized-multihist}.
\end{itemize}

\subsubsection{Histogram Bin Distribution}
\label{hist:qerror:bin-dist}
Our technique allows us to adjust the bin distribution. We choose, as a hyperparameter, the histogram bin distribution to use, \(B\). For a histogram of \(N\) bins, we use the quantile function of \(B\), \(Q_B\), to construct a histogram with \(N\) equally probable bins under the distribution \(B\).

\begin{equation}
      \label{eqn:bin-dist}
      \vec{b} = \left[ 0+\epsilon, Q_B(0), Q_B(1/N), Q_B(2/N), \dots, Q_B(1-1/N), 1-\epsilon \right]
\end{equation}

where \(\epsilon\) is a small constant.

For our experimental results, we found that normally-distributed histogram bins yielded the best results across all evaluated datasets. Ablation studies comparing the performance of uniformly-distributed endpoints with normally-distributed endpoints can be found in \cref{sec:ablation:bin-distribution}.

\subsection{Histogram Distance Bias}
\label{sec:motivation:histogram-distance-bias}

\begin{figure}[h]
      \includegraphics[width=1.0\textwidth]{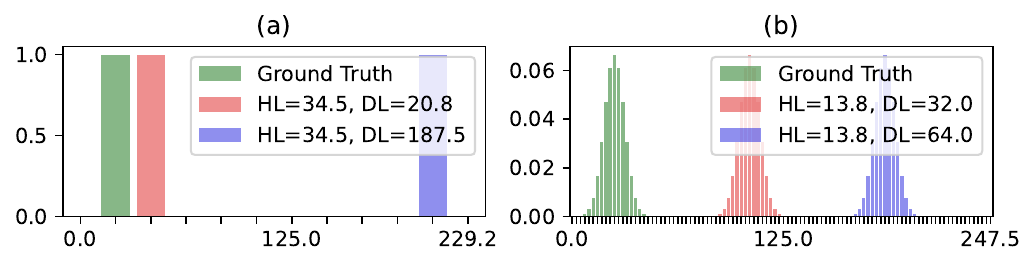}
      \caption{The model's output is represented by the green histograms, and the red and blue histograms represent two sample model predictions. The histogram loss value (HL) and the distance-based loss metric (DL) are shown in the legend box. Note that in (a), the red histogram is much closer to the ground truth than the blue histogram, but their HL values are equal. The distance-based metric fixes this. (b) demonstrates how this phenomenon is present even if we induce a Gaussian distribution.}
      \label{fig:motivation:metric-bias}
\end{figure}

This refers to the error introduced by using cross entropy as our histogram loss function.
Cross entropy treats every bin's probability independently, and thus, cannot distinguish
the magnitude or `distance' of errors in the histogram's prediction (see \cref{fig:motivation:metric-bias} (a)).
Existing implementations of histogram loss~\cite{imani2018improving} utilize an
induced Gaussian representation for the probability histogram to avoid this bias,
but as demonstrated in \cref{fig:motivation:metric-bias} (b), this technique does
not fully address this bias.

To fix this error, we use an additional distance-based term in our overall loss.
This loss computes the $\ell_1$ distance between the expected value of the model's
predicted histogram and the true scalar.

\begin{equation}
      L_{DL}(\hat{Y}, y) = \|y - \mathbb{E}(\hat{Y})\|_1
\end{equation}

Our final loss, then, is the combination of the histogram loss and the distance-based loss.
\begin{equation}
      \label{eq:final-loss}
      L_{\methodacronym{}}(\hat{Y}, y) = \alpha_{HL} L_{HL}(\hat{Y}, \Phi(y)) + \alpha_{DL} L_{DL}(\hat{Y}, y)
\end{equation}
where \(\alpha_{HL}\) and \(\alpha_{DL}\) are the coefficients of the distributional and distance-based losses, respectively. These coefficients are chosen as hyperparameters, and their values can be adjusted during training according to a pre-defined schedule (see \cref{sec:ablation:loss-coefficients} for ablation studies).

\section{Analysis}

\subsection{Stable Gradients}
We prove that \methodacronym{} produces more stable gradients by comparing the bounds of \methodacronym{} gradient norm with those of the MSE loss gradient norm. We use same notation as in \cref{sec:historeg}.



\begin{theorem}
      Assume that \(g(x)\) is locally \(l\)-Lipschitz continuous w.r.t the model's parameters, \(\theta\):
      \begin{equation}
            \left\| \frac{\partial g(x)}{\partial \theta} \right\| \leq l
      \end{equation}
      Then, the norm of the gradient of \methodacronym{} loss w.r.t. \(\theta\) is bounded by:
      \begin{equation}
            \left\| \nabla_\theta L_{\methodacronym{}}(f(x), y) \right\|
            \leq
            l~\left\|\vec{p} - f(x)\right\|
            \left[
            1+
            \sqrt{2} ~
            {\left\|f(x)\right\|}_2
            \left\|\vec{b}\right\|
            \right]
      \end{equation}
\end{theorem}
The proof for this theorem can be found in \cref{appendix:stablegrad:proof}.
We make the following key observations from the gradient norm bound:
\begin{itemize}
      \item The quantity \(\left\|\vec{p} - f(x)\right\|\) has an upper bound of 1. In practice, it should be very small as both \(\vec{p}\) and \(f(x)\) are probability vectors, and, as the model's predictions get closer to the ground truth, this quantity approaches zero.
      \item The quantity \({\left\|f^2(x)\right\|}_2\) is \(\ell_2\)-norm of the model's output probability vector. It is always less than 1, and it heavily depends on our choice of induced probability distribution. The value of this quantity is maximized when the model predicts a categorical (i.e., one-hot) distribution, and is minimized when the model predicts a uniform distribution. A Gaussian-distributed histogram also produces a very small \(\ell_2\)-norm value (see \cref{fig:appendix_distributions_l2}). This is consistent with our experimental observations (see \cref{sec:induced-distribution-ablation}), which show that Gaussian induced distributions are more performant than categorical induced distributions.
      \item The quantity \(\left\|\vec{b}\right\|\) restricts the distance-based portion of the loss to only consider the range that is covered by the histogram. This puts a lower, more restrictive bound on the gradient norm than MSE, which has no such restriction.
\end{itemize}

\Citet{hardt2016train} suggests that lower gradient bounds yield lower generalization errors. We compare the bounds for the \methodacronym{} loss and the MSE loss, which has a gradient norm of \(l \left| f(x) - y \right|\). Using the observations above, we show that the \methodacronym{} loss has a lower gradient norm bound. This is because the gradient norm of MSE is unbounded --- \(f(x)\) could produce any value in \(\mathbb{R}\).For \methodacronym{}, however, the norm of the histogram component is --- at most --- bounded by \(l \left[1 + \sqrt{2} \left(y_{max} - y_{min}\right)\right]\). In reality, however, the values of \(\left\|\vec{p} - f(x)\right\|\) and \({\left\|f^2(x)\right\|}_2\) are much lower. Using the \({\left\|f^2(x)\right\|}_2\) value for the Gaussian distribution in \cref{fig:appendix_distributions_l2}, our gradient norm bound would be \(l \left[1 + \frac{\sqrt{2}}{200} \left(y_{max} - y_{min}\right)\right]\).


\subsection{Uncertainty Quantification}
One natural consequence of predicting probability distributions is that we have a natural way to quantify uncertainty in our model's predictions.
In the single histogram case, this is trivially achieved by computing the entropy of the predicted histogram: \(H(\hat{Y}) = -\sum_k{\hat{Y}_k \log(\hat{Y}_k)}\).

For multiple output heads, we have access to two separate natural uncertainty quantifiers: (1) the mean entropy across all output heads, and (2) the mean KL-divergence across all pairs of output heads. The former is very similar to the single histogram case and is easy to compute.
The latter, however, is more difficult to compute. This is because the KL-divergence requires the two histograms to be defined on the same support, but, by definition, this is not always the case in our method. Therefore, we use a trick to compute the KL-divergence: We pick one predicted histogram, \(\hat{Y}^{(0)}\) and linearly interpolate all predicted histograms to \(\hat{Y}^{(0)}\)'s support. We can then easily compute the KL-divergence between the interpolated histograms: \(H(\hat{Y}^{(i)}, \hat{Y}^{(j)}) = -\sum_k{\tau(\hat{Y}^{(i)}_k) \log{\left(\frac{\tau{\left(\hat{Y}^{(j)}_k\right)}}{\tau{\left(\hat{Y}^{(i)}_k\right)}}\right)}}\), where \(\tau(Y)\) is our interpolation function. Our uncertainty score, then, is the maximum KL divergence between any two histograms.
\begin{equation}
      U(\hat{Y}) = \max_{i,j} \left[H(\hat{Y}^{(i)}, \hat{Y}^{(j)})\right]
\end{equation}

Finally, we must note that both these uncertainty entropy values are in nats and thus must be scaled before being useful. In our method, we use a simple technique, where we use 10\% of the validation to compute optimal shifting and scaling parameters for the uncertainty scores. The final uncertainty score is then:

\begin{equation}
      U_{\text{adjusted}}(\hat{Y}) = \gamma U(\hat{Y}) + \delta
\end{equation}

where \(\gamma, \delta \in \mathbb{R} \) are the scaling and shifting parameters.

\section{Evaluation}
In this section, we will evaluate models trained using \methodacronym{} loss with their baseline (i.e., models trained on L1 loss) performance. For all experiments, we will report the MAE of the targets of interest.
Whenever applicable, we will also report threshold accuracies (WT), which is the percentage of targets that are correctly predicted within a given margin of error. For all experiments, this threshold is set at 0.1\% of the maximum possible error (i.e., \(0.001|y_{max} - y_{min}|\)). Information on the training environment and the amount of compute used during training can be found in \cref{appendix:compute-resources}.


\subsection{Open Catalyst 2020: Relaxed Energy Prediction}
The OC20 dataset~\cite{ocp_dataset}
is one of the largest publicly available catalysis datasets, containing $1.2M$
relaxations of adsorbate-catalyst structures. For our experiments, we evaluate
our technique on the Initial Structure to Relaxed Energy (IS2RE) task.
This is a graph-level regression task in which the model has to estimate the
relaxed energy of an adsorbate-catalyst system from an initial state.
The training dataset comprises of $466k$ initial unrelaxed adsorbate-catalyst
structures paired with relaxed energies.


We evaluate our method with four different backbone GNNs -- SchNet \cite{schutt2017schnet},
DimeNet++~\cite{klicpera2020fast,klicpera2020directional}, GemNet-dT~\cite{klicpera2021gemnet},
and 3D Graphormer~\cite{ying2021transformers,shi2022benchmarking}.
%
%
\Cref{tab:oc20-val} shows validation results for all backbone models when
trained on the full OC20 IS2RE dataset.
For the GemNet and 3D Graphormer models, we evaluate these methods with and
without the relaxed position prediction auxiliary task from \citet{shi2022benchmarking}.
Our results demonstrate some very clear trends.
First, across nearly all backbone models, using \methodacronym{} shows substantial
improvements over the baseline, HL and DL variants.
For 3D Graphormer, the baseline performs about the same as DMoE in energy MAE,
but DMoE performs significantly better on the Energy within Threshold (EwT) metric.
Second, the addition of the relaxed position prediction auxiliary task improves
performance across GemNet-dT and 3D Graphormer, but it does not change
the observed trends on EwT with \methodacronym{}.
Third, \methodacronym{} is not effective on the out-of-distribution adsorbate (OOD Ads)
and OOD Both splits. It neither hurts nor improves performance. We discuss this further
in \cref{sec:eval:effectiveness}.

\begin{table}[!htp]\centering
    \caption{Accuracy scores across the different validation splits in the OC20 IS2RE dataset. For each model, we evaluate using the baseline L1 loss, the histogram loss (HL-only), the distance-based loss (DL-only), and the \methodname{} Loss (\methodacronym{}). Energy MAEs are reported in eV.}\label{tab:oc20-val}
    \scriptsize
    \begin{tabular}{lrrrrrrrrrr}\toprule
                                       &                & \multicolumn{2}{c}{ID} & \multicolumn{2}{c}{OOD Ads} & \multicolumn{2}{c}{OOD Cat} & \multicolumn{2}{c}{OOD Both}                                                                     \\\cmidrule{3-10}
                                       &                & MAE                 & EwT                         & MAE                      & EwT                          & MAE         & EwT            & MAE         & EwT            \\\midrule
        \multirow{4}{*}{SchNet}        & Baseline       & 0.954                  & 1.9\%                       & 1.045                       & 1.5\%                        & 0.925          & 1.8\%          & 0.964          & 1.5\%
        \\
                                       & HL-only        & 0.713                  & 3.0\%                       & 0.910                       & 1.7\%                        & 0.706          & 3.0\%          & 0.822          & 1.8\%
        \\
                                       & DL-only        & 1.780                  & 0.8\%                       & 1.701                       & 0.9\%                        & 1.788          & 0.8\%          & 1.586          & 0.9\%
        \\
                                       & DMoE           & 0.702                  & 3.0\%                       & 0.914                       & 1.7\%                        & 0.695          & 3.3\%          & 0.820          & 1.8\%
        \\
        \hline
        \multirow{4}{*}{DimeNet++}     & Baseline       & 0.675                  & 3.1\%                       & 0.774                       & 2.1\%                        & 0.665          & 3.1\%          & 0.720          & 2.2\%
        \\
                                       & HL-only        & 0.671                  & 3.3\%                       & 0.901                       & 1.9\%                        & 0.647          & 3.8\%          & 0.806          & 2.1\%
        \\
                                       & DL-only        & 1.750                  & 0.8\%                       & 1.863                       & 0.7\%                        & 1.746          & 0.9\%          & 1.696          & 0.8\%
        \\
                                       & DMoE           & 0.652                  & 3.6\%                       & 0.790                       & 2.1\%                        & 0.635          & 4.0\%          & 0.725          & 2.2\%
        \\
        \hline
        \multirow{6}{*}{GemNet-dT}     & Baseline       & 0.610                  & 3.6\%                       & 0.795                       & 2.2\%                        & 0.626          & 3.5\%          & 0.770          & 2.2\%
        \\
                                       & HL-only        & 0.571                  & 5.0\%                       & 0.821                       & 2.0\%                        & 0.579          & 5.2\%          & 0.738          & 2.3\%
        \\
                                       & DL-only        & 2.253                  & 0.8\%                       & 1.940                       & 0.7\%                        & 2.208          & 0.7\%          & 1.674          & 0.9\%
        \\
                                       & DMoE           & 0.557                  & 5.0\%                       & 0.752                       & 2.4\%                        & 0.560          & 4.9\%          & 0.676          & 2.5\%
        \\
                                       & Baseline + Pos & 0.476                  & 5.7\%                       & 0.563                       & 3.5\%                        & 0.482          & 6.0\%          & 0.506          & 3.6\%          \\&DMoE + Pos &0.450 &7.8\% &0.574 &3.5\% &0.459 &7.4\% &0.513 &3.6\%
        \\
        \hline
        \multirow{6}{*}{3D Graphormer} & Baseline       & 0.528                  & 4.9\%                       & 0.705                       & 2.4\%                        & 0.531          & 4.7\%          & 0.645          & 2.7\%
        \\
                                       & HL-only        & 0.560                  & 5.8\%                       & 0.824                       & 2.1\%                        & 0.567          & 5.7\%          & 0.735          & 2.3\%
        \\
                                       & DL-only        & 0.529                  & 5.3\%                       & 0.709                       & 2.6\%                        & 0.531          & 5.2\%          & 0.632          & 2.5\%
        \\
                                       & DMoE           & 0.537                  & 6.0\%                       & 0.785                       & 2.3\%                        & 0.545          & 5.8\%          & 0.709          & 2.5\%
        \\
                                       & Baseline + Pos & 0.451                  & 6.5\%                       & 0.613                       & 2.8\%                        & 0.466          & 6.5\%          & 0.560          & 3.0\%
        \\
                                       & DMoE + Pos     & 0.457                  & 7.3\%                       & 0.712                       & 2.8\%                        & 0.469          & 7.0\%          & 0.623          & 3.2\%
        \\
        \hline
        GemNet*                        & DMoE + Pos     & \textbf{0.389}         & \textbf{10.1\%}             & \textbf{0.544}              & \textbf{4.9\%}               & \textbf{0.395} & \textbf{9.3\%} & \textbf{0.486} & \textbf{5.4\%}
        \\

        \bottomrule
    \end{tabular}
\end{table}

To fully demonstrate the effectiveness of \methodacronym{}, we design a variant of the GemNet model, which we will refer to as GemNet*, with the following changes: 1) We increase the number of blocks to 12 and repeat these blocks 2 times. 2) We use layer norm before each interaction block. 3) Instead of using one output block for each interaction block, we use a single output block at the end of the model to output the energy probability distribution. 4) Finally, we apply Gaussian augmentation to the atoms' positions.
We evaluate this model on the OC20 IS2RE test set.

\Cref{tab:oc20-test} shows results for some of the current state-of-the-art models, as well as other baselines, for the OC20 IS2RE test set.
Similar to the previous results, we see that the \methodacronym{} model performs
significantly well on the ID and OOD Cat splits,
but its performance drops in the OOD Ads and OOD Both splits.

\begin{table}[!htp]\centering
    \caption{OC20 IS2RE test results. Energy MAEs are reported in eV.}\label{tab:oc20-test}
    \scriptsize
    \begin{tabular}{lrrrrrrrrr}\toprule
                                                                               & \multicolumn{2}{c}{ID} & \multicolumn{2}{c}{OOD Ads} & \multicolumn{2}{c}{OOD Cat} & \multicolumn{2}{c}{OOD Both}                                                                       \\\cmidrule{2-9}
                                                                               & MAE                    & EwT                         & MAE                         & EwT                          & MAE            & EwT             & MAE            & EwT             \\\midrule
        GemNet-OC (Relaxation) \cite{gasteiger2022graph}                       & \textbf{0.331}         & \textbf{18.4\%}             & \textbf{0.336}              & \textbf{15.2\%}              & \textbf{0.379} & \textbf{14.2\%} & 0.344          & \textbf{12.1\%}
        \\
        GemNet-XL (Relaxation) \cite{sriram2022towards}                        & 0.376                  & 13.3\%                      & 0.368                       & 10.0\%                       & 0.402          & 11.6\%          & \textbf{0.338} & 9.7\%
        \\
        GemNet-T (Relaxation) \cite{klicpera2021gemnet}                        & 0.390                  & 12.4\%                      & 0.391                       & 9.1\%                        & 0.434          & 10.1\%          & 0.384          & 7.9\%
        \\
        SpinConv (Relaxation) \cite{shuaibi2021rotation}                       & 0.421                  & 9.4\%                       & 0.438                       & 7.5\%                        & 0.458          & 8.2\%           & 0.420          & 6.6\%
        \\
        DimeNet++ (Relaxation) \cite{klicpera2020directional,klicpera2020fast} & 0.503                  & 6.6\%                       & 0.543                       & 4.3\%                        & 0.579          & 5.1\%           & 0.611          & 3.9\%
        \\
        \hline
        CGCNN (Direct) \cite{xie2018crystal}                                   & 0.615                  & 3.4\%                       & 0.916                       & 1.9\%                        & 0.622          & 3.1\%           & 0.851          & 2.0\%
        \\
        SchNet (Direct) \cite{schutt2017schnet}                                & 0.639                  & 3.0\%                       & 0.734                       & 2.3\%                        & 0.662          & 2.9\%           & 0.704          & 2.2\%
        \\
        NequIP (Direct) \cite{batzner2021se}                                   & 0.602                  & 3.2\%                       & 0.784                       & 2.2\%                        & 0.619          & 3.1\%           & 0.736          & 2.1\%
        \\
        PaiNN~ (Direct) \cite{schutt2021equivariant}                           & 0.575                  & 3.5\%                       & 0.783                       & 2.0\%                        & 0.604          & 3.5\%           & 0.743          & 2.3\%
        \\
        DimeNet++ (Direct) \cite{klicpera2020directional,klicpera2020fast}     & 0.562                  & 4.3\%                       & 0.725                       & 2.1\%                        & 0.576          & 4.1\%           & 0.661          & 2.4\%
        \\
        SphereNet (Direct) \cite{liu2021spherical}                             & 0.563                  & 4.5\%                       & 0.703                       & 2.3\%                        & 0.571          & 4.1\%           & 0.638          & 2.4\%
        \\
        SEGNN (Direct) \cite{brandstetter2021geometric}                        & 0.533                  & 5.4\%                       & 0.692                       & 2.5\%                        & 0.537          & 4.9\%           & 0.679          & 2.6\%
        \\
        Noisy Nodes (Direct) \cite{godwin2021very}                             & 0.422                  & 9.1\%                       & \textbf{0.568}              & \textbf{4.3\%}               & 0.437          & 8.0\%           & \textbf{0.465} & \textbf{4.6\%}
        \\
        3D Graphormer (Direct) \cite{ying2021transformers,shi2022benchmarking} & 0.398                  & 9.0\%                       & 0.572                       & 3.5\%                        & 0.417          & 8.2\%           & 0.503          & 3.8\%
        \\
        GemNet*+DMoE (Direct, Ours)                                            & \textbf{0.390}         & \textbf{10.1\%}             & 0.640                       & 3.6\%                        & \textbf{0.401} & \textbf{8.9\%}  & 0.576          & 3.9\%
        \\
        \bottomrule
    \end{tabular}

\end{table}

\subsection{MD17: Molecular Dynamics}
MD17~\cite{chmiela2017machine} contains energies and forces for molecular dynamics trajectories of eight organic molecules. The task is two-fold: (1) predicting the energy of a molecular system, and (2) predicting the force of each atom in the system. For our experiments, we demonstrate results on the 1k and 50k splits of the MD17 dataset. However, while the 50k split provides much more training data, it does not guarantee independent samples in the test set and is therefore much less reliable.

We evaluate DMoE using two separate backbone GNNs, one for each method of estimating forces --
SchNet with forces calculated as negative gradient of the energy w.r.t atom positions, and
GemNet-dT with the direct-force output head.
For each, we report the energy and force MAE values for the baseline (trained on L1 loss) with the \methodacronym{} method applied to the energy outputs.

\begin{table}[!htp]\centering
    \caption{Validation scores across all different molecules in the 1k and 50k splits of the MD17 dataset. Energy MAEs are reported in meV, and force MAEs are reported in meV/\si{\angstrom}.}\label{tab:md17}
    \scriptsize
    \begin{tabular}{lrrrrrrrrr}\toprule
                            & \multicolumn{4}{c}{GemNet-dT} & \multicolumn{4}{c}{SchNet}                                                                                                                                       \\\cmidrule{2-9}
                            & \multicolumn{2}{c}{Baseline}  & \multicolumn{2}{c}{\methodacronym{}} & \multicolumn{2}{c}{Baseline} & \multicolumn{2}{c}{\methodacronym{}}                                                           \\\cmidrule{2-9}
                            & Energy                        & Forces                             & Energy                       & Forces                             & Energy & Forces        & Energy         & Forces        \\\midrule
        Aspirin (1k)        & \textbf{105.3}                & 45.6                               & 109.9                        & \textbf{45.5}                      & 152.8  & \textbf{72.2} & \textbf{149.7} & 79.8          \\
        Benzene (1k)        & 27.9                          & \textbf{14.0}                      & \textbf{21.9}                & 15.3                               & 70.5   & \textbf{18.3} & \textbf{70.4}  & 18.7          \\
        Ethanol (1k)        & \textbf{19.8}                 & \textbf{23.9}                      & 24.1                         & 25.1                               & 116.6  & \textbf{38.3} & \textbf{116.5} & 42.0          \\
        Malonaldehyde (1k)  & \textbf{52.2}                 & \textbf{42.8}                      & 56.1                         & 44.1                               & 113.2  & \textbf{62.0} & \textbf{112.7} & 64.0          \\
        Naphthalene (1k)    & 59.6                          & \textbf{23.3}                      & \textbf{54.7}                & 30.9                               & 146.4  & \textbf{39.5} & \textbf{146.2} & 42.9          \\
        Salicylic (1k)      & \textbf{75.3}                 & \textbf{39.8}                      & 83.4                         & 41.4                               & 144.7  & 61.9          & \textbf{143.6} & \textbf{60.8} \\
        Toluene (1k)        & 58.8                          & \textbf{25.2}                      & \textbf{56.8}                & 32.6                               & 132.0  & \textbf{38.9} & \textbf{131.4} & 42.6          \\
        Uracil (1k)         & \textbf{54.9}                 & \textbf{36.6}                      & 60.7                         & 39.5                               & 127.7  & \textbf{56.2} & \textbf{127.3} & 57.4          \\
        \hline
        Aspirin (50k)       & 6.9                           & 3.3                                & \textbf{5.0}                 & \textbf{2.9}                       & 151.8  & \textbf{24.2} & \textbf{150.7} & 26.0          \\
        Benzene (50k)       & \textbf{1.3}                  & \textbf{6.0}                       & 1.4                          & 6.3                                & 74.3   & \textbf{11.1} & \textbf{74.0}  & 14.2          \\
        Ethanol (50k)       & \textbf{0.6}                  & \textbf{0.9}                       & 1.0                          & 1.3                                & 119.9  & \textbf{7.0}  & \textbf{119.4} & 9.9           \\
        Malonaldehyde (50k) & \textbf{1.0}                  & \textbf{1.2}                       & 1.5                          & 2.2                                & 119.7  & \textbf{13.0} & \textbf{119.3} & 16.2          \\
        Naphthalene (50k)   & 1.8                           & \textbf{1.2}                       & \textbf{1.7}                 & 1.3                                & 142.9  & \textbf{12.8} & \textbf{142.1} & 15.7          \\
        Salicylic (50k)     & 3.8                           & \textbf{2.4}                       & \textbf{3.2}                 & 2.6                                & 140.5  & \textbf{18.7} & \textbf{140.0} & 22.5          \\
        Toluene (50k)       & 1.6                           & \textbf{1.1}                       & \textbf{1.4}                 & \textbf{1.1}                       & 128.2  & \textbf{14.2} & \textbf{127.1} & 18.9          \\
        Uracil (50k)        & 1.3                           & \textbf{1.3}                       & \textbf{1.2}                 & 1.4                                & 128.8  & \textbf{14.1} & \textbf{127.4} & 17.7          \\
        \bottomrule
    \end{tabular}
\end{table}

\Cref{tab:md17} shows the MD17 results for the SchNet and GemNet-dT backbone models.
For the SchNet model, \methodacronym{} performs better on energy predictions but worse on force predictions. We believe that the decrease in force prediction performance is due to the fact that SchNet calculates forces by differentiating the energy with respect to atom positions, and, as a result, any quantization errors in the energy predictions will also propagate to the force predictions.
Similarly for the GemNet model, we see improved energy prediction performance across about half of the molecules but decreased force prediction performance overall.

\subsection{QM9: Molecular Property Prediction}
QM9~\cite{doi:10.1021/ci300415d, ramakrishnan2014quantum} is a dataset which contains quantum chemical properties of 134,000 organic molecules made up of CHNOF atoms. For this dataset, we train the MXMNet~\cite{zhang2020molecular} as our backbone. To keep the methodology consistent with MXMNet, we only use atomization energies for \(U_0\), \(U\), \(H\), and \(G\). For each target property, we train a separate model for 100 epochs and report the best validation results.

\begin{table}[t]
    \begin{minipage}[t]{0.48\linewidth}
          \resizebox{0.8\textwidth}{!}{
                \centering
                \begin{minipage}{\textwidth}
                    \centering
                    \captionsetup{justification=justified}
                    \caption{Validation scores for the MXMNet backbone model on the QM9 dataset}
                      \begin{tabular}{lrr|rrr}\toprule
                                                  & \multicolumn{2}{c}{Baseline} & \multicolumn{2}{c}{\methodacronym{}}                                     \\\cmidrule{2-5}
                                                  & MAE                          & WT                                   & MAE             & WT              \\\midrule
                            \(\mu\) (D)           & 0.0444                       & 56.9\%                               & \textbf{0.0397} & \textbf{60.9\%} \\
                            \(a\) (\(a^3_{0}\))   & \textbf{0.0624}              & \textbf{95.9\%}                      & 0.0716          & 95.1\%          \\
                            \(e_{HOMO}\) (eV)     & 0.0290                       & 26.5\%                               & \textbf{0.0275} & \textbf{33.1\%} \\
                            \(e_{LUMO}\) (eV)     & 0.0227                       & \textbf{40.2\%}                      & \textbf{0.0225} & 39.9\%          \\
                            \(\Delta e\) (eV)     & 0.0456                       & 29.8\%                               & \textbf{0.0414} & \textbf{36.4\%} \\
                            \(R^2\) (\(a^2_{0}\)) & 1.3034                       & 93.9\%                               & \textbf{0.9813} & \textbf{95.9\%} \\
                            \(ZPVE\) (eV)         & 0.0015                       & 99.1\%                               & \textbf{0.0014} & \textbf{99.5\%} \\
                            \(U_{0}\) (eV)        & \textbf{0.0226}              & \textbf{98.9\%}                      & 0.0236          & \textbf{98.9\%} \\
                            \(U\) (eV)            & \textbf{0.0143}              & \textbf{99.7\%}                      & 0.0235          & 98.9\%          \\
                            \(H\) (eV)            & \textbf{0.0144}              & \textbf{99.6\%}                      & 0.0228          & 99.0\%          \\
                            \(G\) (eV)            & \textbf{0.0138}              & \textbf{99.6\%}                      & 0.0223          & 98.7\%          \\
                            \(c_{v}\) (cal/mol.K) & \textbf{0.0253}              & \textbf{82.7\%}                      & 0.0272          & 80.5\%          \\
                            \bottomrule
                      \end{tabular}
                      \label{tab:qm9}
                \end{minipage}
          }
    \end{minipage}
    \quad
    \begin{minipage}[t]{0.48\linewidth}
          \resizebox{0.8\textwidth}{!}{
                \centering
                \begin{minipage}{\textwidth}
                    \centering
                    \captionsetup{justification=justified}
                      \caption{Uncertainty metrics for GemNet-dT evaluated on the OC20 ID validation split}
                      \begin{tabular}{lrrrr}\toprule
                                                                & MACE            & RMSCE           & MA              \\\midrule
                            Ensemble                            & 0.2752          & 0.3119          & 0.2780          \\
                            Ensemble + IR Recal                 & 0.0237          & 0.0334          & 0.0240          \\
                            \methodacronym{}-Entropy            & 0.0427          & 0.0668          & 0.0431          \\
                            \methodacronym{}-Entropy + IR Recal & 0.0116          & 0.0303          & 0.0117          \\
                            \methodacronym{}-KL                 & 0.0542          & 0.0664          & 0.0547          \\
                            \methodacronym{}-KL + IR Recal      & \textbf{0.0057} & \textbf{0.0165} & \textbf{0.0057} \\
                            \bottomrule
                      \end{tabular}
                      \label{tab:uncertainty}
                \end{minipage}
          }
    \end{minipage}
    \vspace{-20pt}
\end{table}
\Cref{tab:qm9} shows QM9 validation scores for the MXMNet backbone model. We see very promising results --- with at least a 4\% improvement in the threshold metric --- across multiple different targets, including \(\mu\), \(e_{HOMO}\), \(e_{LUMO}\) and \(\Delta e\). For the rest of the targets, both models seem to produce similarly accurate results, from a threshold metric perspective.

\subsection{Uncertainty Quantification}
To evaluate the reliability of the uncertainties computed, we compare the uncertainties predicted by GemNet* to the uncertainties predicted by an ensemble of 5 baseline GemNet-dT models, both trained with the position-prediction auxiliary task. All of these models are trained on the full OC20 IS2RE dataset and are evaluated on the ID validation split.
For each model, we report the mean absolute calibration error (MACE), the root mean squared calibration error (RMSE), and the miscalibration area (MA).
All methods are shown with and without isotonic regression recalibration technique proposed by~\citet{kuleshov2018accurate}. For this purpose, we set aside 10\% of the evaluation data split for recalibration.
For the \methodacronym{} uncertainty methods, the data used for isotonic regression recalibration and for our readjustment procedure are the same (\ie, the same 10\% of the evaluation data split).

\Cref{tab:uncertainty} shows these uncertainty results. There are two key observations from these results:
1) The \methodacronym{} KL method provides excellent uncertainty metrics, outperforming all other methods.
2) While the \methodacronym{} entropy and KL methods can provide acceptable uncertainty metrics without the need for isotonic regression recalibration, performing this recalibration further the uncertainty estimates.


\subsection{Ablations}
\subsubsection{Bin Distribution}
\label{sec:ablation:bin-distribution}

We investigate the impact of changing the bin distribution hyperparameters. Specifically, we experiment with changing the number of histograms, the number of bins in each histogram, and the bin distribution (\ie, uniform or Gaussian).
All experiments are done on the \methodacronym{}+Pos version of the model from \cref{tab:oc20-val} and are evaluated on the ID validation split of the OC20 dataset.
\Cref{tab:bindist-ablation} shows these results. We see, clearly, that using a Gaussian bin distribution produces substantially better results.
Increasing the number of histograms seems to help with the uniform bin distributions, but not with the Gaussian bin distribution.
Conversely, increasing the number of bins in each histogram seems to help with the Gaussian bin distribution, but not with the uniform bin distribution.

\subsubsection{Loss Coefficients}
\label{sec:ablation:loss-coefficients}
We investigate the effect of the histogram and distance-based loss coefficients.
To do this, we compute a normalizing constant such that the magnitude of the HL and DL losses are equal. Then, we experiment with different values for these coefficients. The results can be found in \cref{tab:alpha-beta-ablation}.
Our primary observation was that a higher histogram loss coefficient \(\alpha_{HL}\) yields better results and faster training at the beginning of training, but as training progresses, the biases outlined in \cref{sec:historeg} begin to show their effects.
To demonstrate this, we launch a variant where we schedule the loss coefficients from \(\alpha_{HL} = 0.9, \alpha_{DL} = 0.1\) to \(\alpha_{HL} = 0.05, \alpha_{DL} = 0.95\) over the course of 20 epochs. The results for this run are displayed on the ``Scheduled'' row of \cref{tab:alpha-beta-ablation}.

\subsection{Analysis}

\begin{table}[t]
    \begin{minipage}[t]{0.48\linewidth}
          \resizebox{0.85\textwidth}{!}{
                \centering
                \begin{minipage}{\textwidth}
                    \centering
                    \captionsetup{justification=justified}
                    \caption{The impact of different HL and DL coefficients on the performance of the GemNet+Pos model on the ID split of OC20}

                    \begin{tabular}{lrrrr}\toprule
                        \(\alpha_{HL}\)               & \(\alpha_{DL}\) & E-MAE & EWT            \\\midrule
                        1.00                          & 0.00            & 0.455 & \textbf{8.3\%} \\
                        0.90                          & 0.10            & 0.457 & 8.0\%          \\
                        0.75                          & 0.25            & 0.459 & 8.1\%          \\
                        0.50                          & 0.50            & 0.457 & 8.1\%          \\
                        0.25                          & 0.75            & 0.462 & 7.9\%          \\
                        0.10                          & 0.90            & 0.473 & 7.3\%          \\
                        0.00                          & 1.00            & 2.253 & 0.8\%          \\
                        \multicolumn{2}{c}{Scheduled} & \textbf{0.450}  & 7.8\%                  \\
                        \bottomrule
                    \end{tabular}
                    \label{tab:alpha-beta-ablation}
                \end{minipage}
          }
    \end{minipage}
    \quad
    \begin{minipage}[t]{0.48\linewidth}
          \resizebox{0.85\textwidth}{!}{
                \centering
                \begin{minipage}{\textwidth}
                    \centering
                    \captionsetup{justification=justified}
                    \caption{Validation results for running GemNet+Pos on the ID split of the OC20 IS2RE dataset, varying the bin distribution, number of histograms, and number of bins}
                    \begin{tabular}{lrrrr}\toprule
                                                & \(N_{\text{H}} \text{ x } N_{\text{B}}\) & Energy MAE     & EWT            \\\midrule
                        \multirow{3}{*}{Uniform}  & 32x2048                                  & 0.467          & 6.5\%
                        \\
                                                & 64x1024                                  & \textbf{0.464} & \textbf{7.2}\%
                        \\
                                                & 256x256                                  & 0.466          & 7.1\%
                        \\
                        \hline
                        \multirow{3}{*}{Gaussian} & 32x2048                                  & \textbf{0.454} & \textbf{8.3}\%
                        \\
                                                & 64x1024                                  & 0.455          & 7.9\%
                        \\
                                                & 256x256                                  & 0.459          & 7.3\%
                        \\
                        \bottomrule
                    \end{tabular}
                    \label{tab:bindist-ablation}
                \end{minipage}
          }
    \end{minipage}
    \vspace{-15pt}
\end{table}

\subsubsection{Effectiveness on OOD data}
\label{sec:eval:effectiveness}
One apparent drawback of our technique seems to be its lack of performance on OOD adsorbate and OOD both data on the OC20 dataset (Table 2). We attribute this to two reasons:
\begin{enumerate}
      \item Direct models, in general, seem to perform poorly on OOD adsorbate and OOD both. This might be because relaxation-based models have more physical intuition embedded in them, whereas the direct models tend to overfit to the distribution they're trained on.
      \item Our method defines the histogram bin distribution over the range of possible target values in the train set. This selection of bin distribution introduces a bias in the model. As a result, if the OOD data is not distributed similarly to the training data, the model's histogram distribution will be suboptimal. This argument can be supported by examining the target distribution histograms of the train and OOD splits (see \cref{sec:appendix:dataset_histograms}).
\end{enumerate}


\subsubsection{Quantization Error and Impact of Distance-Based Loss}
\label{sec:eval:quantization}
Comparing MAE and EwT in \Cref{tab:oc20-val}, we notice that the histogram loss leads to
significantly better EwT but worse MAE. We suspect that this is due to the quantization error described in \cref{sec:distribution-quantization-error}. Adding the distance-based loss (the \methodacronym{} rows in \cref{tab:oc20-val}) shows improvements in both MAE and EwT. Scheduling the loss coefficients, as shown in \cref{sec:ablation:loss-coefficients}, attains the best results, suggesting that the distance-based loss contributes to the model's ability to learn the distribution of the target values.

\section{Conclusion}
\label{sec:conclusion}
In this paper, we proposed \methodname{}, a method for enhancing regression performance by learning distributional representations of target values. We show that this method is effective for molecular regression tasks, demonstrating consistent improvements over the baselines across various datasets
(OC20, MD17, QM9) and backbone GNN architectures.

Developing better methods for molecular regression tasks can have many positive consequences, enabling progress in the many subfields of chemistry, such as drug discovery and synthesis and development and discovery of new battery technologies. However, unintended consequences of ML for molecular property prediction must also be considered -- hypothetically, similar technologies can be used to study explosive materials and their properties, which can be used to develop weapons.

\bibliography{reference}

\appendix
\section{Appendix}

\subsection{Implementation, Experimental Setup, and Train Times}
\label{appendix:compute-resources}
In this section, we will describe the training hardware and the total training time for all trained models. All models are programmed in Python using the PyTorch library~\citeappendix{NEURIPS2019_9015}. All model and training code will be open sourced with MIT license upon acceptance.

\subsubsection{GemNet Auxiliary Position}
For our GemNet-dT~\citeappendix{klicpera2021gemnet} experiments, we experiment with a variant that uses the auxiliary position prediction task from \citetappendix{ying2021transformers}. To do this, we re-purpose the force output head of GemNet-dT to predict positions instead. This deviates from the original author's implementation, but it yields improved predictions (see \cref{tab:oc20-val}).

\subsubsection{PaiNN and NequIP Models}
Our PaiNN~\citeappendix{schutt2021equivariant} and NequIP~\citeappendix{batzner2021se} implementations contain one minor difference from the original implementation: The data normalization scheme is different to account for the adsorption energy reference used in OC20~\citeappendix{ocp_dataset}.

\subsubsection{OC20 Models}
\Cref{tab:appendix:oc20-train-times} shows the training times (in hours) for the models in \cref{tab:oc20-val}. 
All models are trained on 16 NVIDIA Tesla V100 Volta 32 GB GPUs.
Training is stopped either after 100 epochs or once the model converges (\ie, using early stopping).
The training time for GemNet* model in \cref{tab:oc20-test} is the same as the one shown for GemNet* in \cref{tab:appendix:oc20-train-times}.

\begin{table}[!htp]\centering
\caption{Training times (in hours) for the models in \cref{tab:oc20-val}}\label{tab:appendix:oc20-train-times}
\scriptsize
\begin{tabular}{lrrr}\toprule
& &Train Time (Hours) \\\midrule
\multirow{4}{*}{SchNet} &Baseline &4.40 \\
&HL &3.28 \\
&DL &0.29 \\
&DMoE &4.61 \\
\multirow{4}{*}{DimeNet++} &Baseline &14.58 \\
&HL &12.56 \\
&DL &0.37 \\
&DMoE &12.60 \\
\multirow{6}{*}{GemNet-dT} &Baseline &5.07 \\
&HL &3.18 \\
&DL &5.34 \\
&DMoE &2.54 \\
&Baseline + Pos &7.12 \\
&DMoE + Pos &16.40 \\
\multirow{6}{*}{3D Graphormer} &Baseline &63.35 \\
&HL &13.26 \\
&DL &63.53 \\
&DMoE &38.49 \\
&Baseline + Pos &81.55 \\
&DMoE + Pos &88.64 \\
GemNet* &DMoE + Pos &115.28 \\
\bottomrule
\end{tabular}
\end{table}

For \cref{tab:uncertainty}, the total train time of the GemNet*-DMoE model is 115.28 hours, and the total train time of the baseline GemNet ensemble is 662.37 hours.

\Cref{tab:appendix:alpha-beta-train-times} shows the train times for the models of \cref{tab:alpha-beta-ablation}.
\begin{table}[!htp]\centering
\caption{Training times (in hours) for the models in \cref{tab:alpha-beta-ablation}}\label{tab:appendix:alpha-beta-train-times}
\scriptsize
\begin{tabular}{lrrr}\toprule
$\alpha$ &$\beta$ &Train Time \\\midrule
1.00 &0.00 &16.24 \\
0.90 &0.10 &9.32 \\
0.75 &0.25 &9.67 \\
0.50 &0.50 &11.06 \\
0.25 &0.75 &11.45 \\
0.10 &0.90 &4.18 \\
0.00 &1.00 &- \\
\multicolumn{2}{c}{Scheduled} &16.40 \\
\bottomrule
\end{tabular}
\end{table}

\Cref{tab:appendix:bindist-train-times} shows the train times for the models of \cref{tab:bindist-ablation}.
\begin{table}[!htp]\centering
\caption{Training times (in hours) for the models in \cref{tab:bindist-ablation}}\label{tab:appendix:bindist-train-times}
\scriptsize
\begin{tabular}{lrrr}\toprule
& &Train Time \\\midrule
\multirow{3}{*}{Uniform} &32x2048 &7.35 \\
&64x1024 &8.86 \\
&256x256 &6.28 \\
\multirow{3}{*}{Normal} &32x2048 &6.98 \\
&64x1024 &5.97 \\
&256x256 &5.41 \\
\bottomrule
\end{tabular}
\end{table}

\Cref{tab:appendix:induced-bin-dist-train-times} shows the train times for the models of \cref{tab:appendix:induced-dist}.
\begin{table}[!htp]\centering
    \caption{Training times (in hours) for the models in \cref{tab:appendix:induced-dist}}\label{tab:appendix:induced-bin-dist-train-times}
    \scriptsize
    \begin{tabular}{lrrrr}\toprule
    & &Normal Bin Distribution &Uniform Bin Distribution \\\midrule
    \multirow{6}{*}{Normal} &$\sigma=0.5 \mathbb{E}{\left[w\right]}$ &8.26 &9.27 \\
    &$\sigma=1 \mathbb{E}{\left[w\right]}$ &10.24 &8.12 \\
    &$\sigma=1.5 \mathbb{E}{\left[w\right]}$ &9.30 &6.60 \\
    &$\sigma=2 \mathbb{E}{\left[w\right]}$ &14.99 &6.05 \\
    &$\sigma=5 \mathbb{E}{\left[w\right]}$ &6.78 &6.65 \\
    &$\sigma=10 \mathbb{E}{\left[w\right]}$ &6.81 &7.45 \\
    Laplace &$b=1 \mathbb{E}{\left[w\right]}$ &10.79 &3.29 \\
    K-Categorical &$k=3$ &5.72 &5.37 \\
    Categorical & &16.01 &7.56 \\
    \bottomrule
    \end{tabular}
\end{table}

\Cref{tab:appendix:robust-train-times} shows the train times for the models of \cref{tab:appendix:robust-losses}.
\begin{table}[!htp]\centering
\caption{Training times (in hours) for the models in \cref{tab:appendix:robust-losses}}\label{tab:appendix:robust-train-times}
\scriptsize
\begin{tabular}{lrr}\toprule
&Train Time \\\midrule
L1 Loss &6.58 \\
L2 Loss &6.62 \\
Smooth L1 Loss &7.12 \\
DMoE &21.52 \\
\bottomrule
\end{tabular}
\end{table}

\subsubsection{MD17 Models}
\Cref{tab:appendix:md17-train-times} shows the training times (in hours) for the models in \cref{tab:md17}. 
All models are trained on a single NVIDIA Tesla V100 Volta 32 GB GPU.
Training is stopped either after 100 epochs or once the model converges (\ie, using early stopping).

\begin{table}[!htp]\centering
\caption{Training times (in hours) for models in \cref{tab:md17}}\label{tab:appendix:md17-train-times}
\scriptsize
\begin{tabular}{lrrrrr}\toprule
&\multicolumn{2}{c}{GemNet} &\multicolumn{2}{c}{SchNet} \\\cmidrule{2-5}
&Baseline &Distributional &Baseline &Distributional \\\midrule
Aspirin (1k) &0.47 &0.50 &0.49 &0.50 \\
Benzene (1k) &0.83 &0.60 &0.48 &0.62 \\
Ethanol (1k) &1.62 &1.80 &0.44 &0.53 \\
Malonaldehyde (1k) &0.66 &0.62 &0.46 &0.59 \\
Naphthalene (1k) &0.82 &0.52 &0.46 &0.54 \\
Salicylic (1k) &0.51 &0.49 &0.48 &0.54 \\
Toluene (1k) &0.66 &0.54 &0.48 &0.69 \\
Uracil (1k) &0.64 &0.50 &0.49 &0.53 \\
\hline
Aspirin (50k) &15.62 &24.28 &14.96 &16.84 \\
Benzene (50k) &25.77 &12.70 &15.46 &14.85 \\
Ethanol (50k) &28.10 &17.24 &13.73 &16.01 \\
Malonaldehyde (50k) &18.06 &17.01 &14.47 &14.32 \\
Naphthalene (50k) &16.78 &33.13 &15.30 &15.97 \\
Salicylic (50k) &26.41 &29.96 &14.84 &15.19 \\
Toluene (50k) &28.61 &33.58 &14.92 &13.15 \\
Uracil (50k) &26.94 &27.74 &14.51 &13.93 \\
\bottomrule
\end{tabular}
\end{table}

\subsubsection{QM9 Models}
\Cref{tab:appendix:qm9-train-times} shows the training times (in hours) for the models in \cref{tab:qm9}.
All models are trained on a single NVIDIA Tesla V100 Volta 32 GB GPU.
Training is stopped either after 100 epochs or once the model converges (\ie, using early stopping).

\begin{table}[!htp]\centering
    \caption{Training times (in hours) for the models in \cref{tab:qm9}}\label{tab:appendix:qm9-train-times}
    \scriptsize
    \begin{tabular}{lrrr}\toprule
    &Baseline &Distributional \\\midrule
    \(\mu\) (D) &4.58 &6.29 \\
    \(a\) (\(a^3_{0}\)) &5.07 &6.04 \\
    \(e_{HOMO}\) (eV) &5.36 &5.46 \\
    \(e_{LUMO}\) (eV) &4.59 &5.77 \\
    \(\Delta e\) (eV) &5.27 &5.79 \\
    \(R^2\) (\(a^2_{0}\)) &4.73 &5.65 \\
    \(ZPVE\) (eV) &4.99 &5.81 \\
    \(U_{0}\) (eV) &5.95 &5.77 \\
    \(U\) (eV) &4.70 &5.63 \\
    \(H\) (eV) &4.54 &5.63 \\
    \(G\) (eV) &5.01 &5.59 \\
    \(c_{v}\) (cal/mol.K) &4.46 &5.61 \\
    \bottomrule
    \end{tabular}
\end{table}

\subsection{Additional Biases}
\label{appendix:biases}
In addition to the biases outlined in \cref{sec:historeg}, we identify the following additional biases in the histogram regression technique in this section.

\subsubsection{Induced Distribution Bias}
This refers to the bias of using a Gaussian distribution (or any other induced distribution) to represent the underlying scalar target value. In other words, this is the bias of assuming some distributional representation of a scalar target. In cases where this bias is present, we may observe a decreasing cross-entropy loss but an increasing L1 loss. 

Our distance-based loss term helps alleviate this bias. Moreover, the $\alpha_{HL}$ and $\alpha_{DL}$ coefficients can be adjusted to minimize this bias. Empirically, we found that starting scheduling the coefficients such that $\alpha_{HL}$ is higher at the beginning of training (\eg, $\alpha_{HL} = 2.0$ and $\alpha_{DL} = 1.0$) and falls to a very low value as training progresses (\eg, $\alpha_{HL} = 0.1$ and $\alpha_{DL} = 1.0$) produces the best results (see \cref{tab:alpha-beta-ablation}).

\begin{figure*}[h]
      \includegraphics[width=1.0\textwidth]{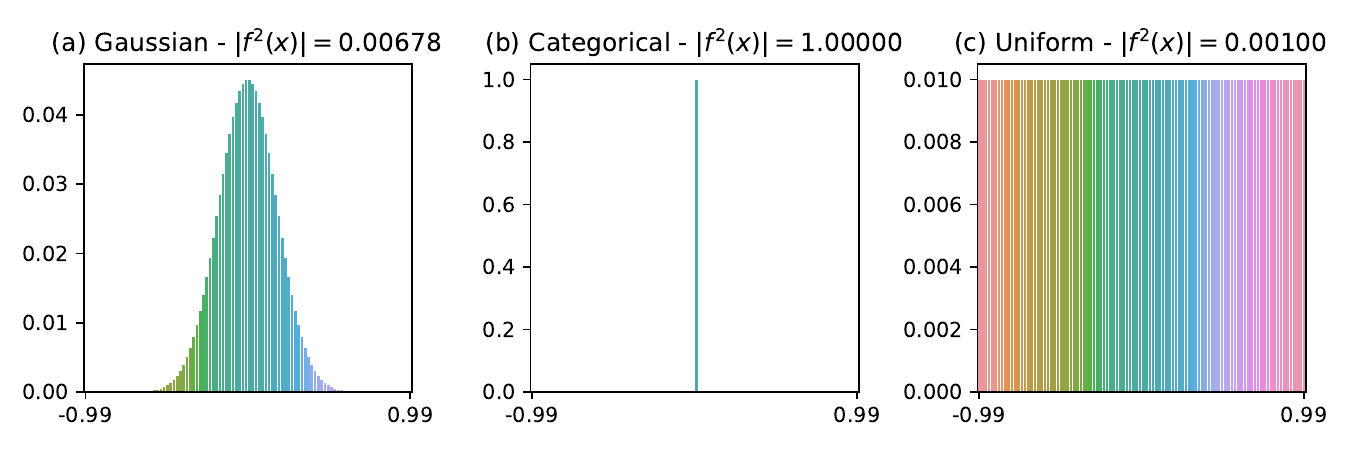}
      \caption{\({\left\|f^2(x)\right\|}_2\) values for example Gaussian, Categorical, and Uniform distributions.}
      \label{fig:appendix_distributions_l2}
\end{figure*}

\subsection{Quantization Error Formula}
To calculate the distribution quantization error of a histogram, shown in \cref{fig:method:normal-bin-dist}, we use the following formula:

\begin{equation}
      \label{eq:quantization-error}
      E_{Q}(Y, B) = \int_{B_i}^{B^{i + 1}}{\left|f(x) - Y_i\right| \cdot f(x) ~ dx}
\end{equation}

where \(f(x)\) is the induced distribution's continuous PDF, \(B\) is the histogram's bin endpoints, and \(Y\) is the bin values for the histogram we are evaluating.

\subsection{Choosing Bin Distributions for Multiple Histograms}
\begin{figure*}[h]
    \includegraphics[width=1.0\textwidth]{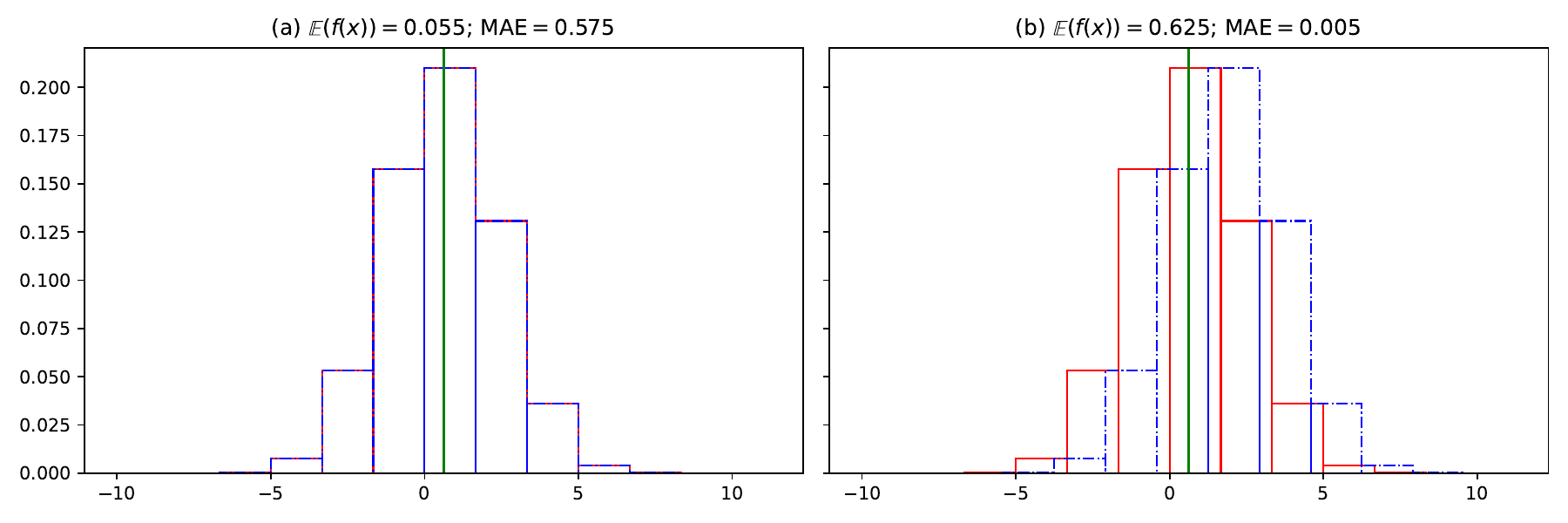}
    \caption{For each figure, the two output histograms are displayed in red and blue. The ground-truth value is displayed by the green line. (a) shows a multi-histogram scenario with uniform bin distributions where the two output histograms share the bin endpoint. (b) shows a multi-histogram scenario where the two output histograms' bin endpoint distributions are adjusted using \cref{alg:appendix:multihist}.}
    \label{fig:appendix:multiple-hist-uniform}
\end{figure*}
\begin{figure*}[h]
    \includegraphics[width=1.0\textwidth]{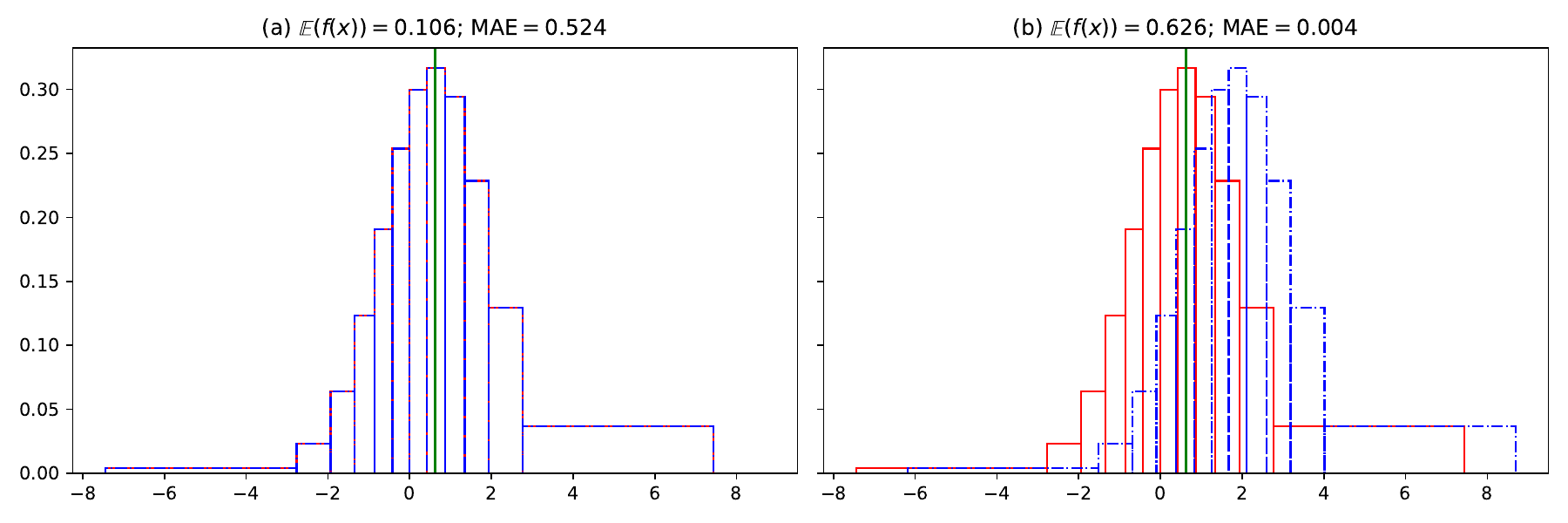}
    \caption{For each figure, the two output histograms are displayed in red and blue. The ground-truth value is displayed by the green line. (a) shows a multi-histogram scenario with normal bin distributions where the two output histograms share the bin endpoint. (b) shows a multi-histogram scenario where the two output histograms' bin endpoint distributions are adjusted using \cref{alg:appendix:multihist}.}
    \label{fig:appendix:multiple-hist-norm}
\end{figure*}

When we use multiple output histograms, simply defining new histograms on the same bin distribution  (as shown in \cref{fig:appendix:multiple-hist-uniform} (a) and \cref{fig:appendix:multiple-hist-norm} (b)) does not help reduce our quantization error. Instead, we must vary each histogram's bin distribution so that each new output head helps reduce our quantization error. \Cref{fig:appendix:multiple-hist-uniform} (b) shows an example of this for uniformly distributed bin endpoints. Similarly, \cref{fig:appendix:multiple-hist-norm} (b) shows an example of this for normally distributed bin endpoints. For both examples, the MAE between the expected value of the histogram and the ground-truth value is shown. We see that in both cases, simply shifting the bin distribution for the second histogram massively reduces this MAE.

To expand our algorithm to support multiple bin endpoints, we begin with the single endpoint scenario. We then slightly offset the $i^{\text{th}}$ histogram's bin endpoints by \(\frac{i}{M} \mathbb{E}(w)\), where \(M\) is the number of histograms in our output head. Concretely, for multiple histograms, we let \(b_0\) be the bin distribution as defined in \cref{eqn:bin-dist}:
\begin{equation}
      \vec{b}_0 = \left[ 0+\epsilon, Q_B(0), Q_B(1/N), Q_B(2/N), \dots, Q_B(1-1/N), 1-\epsilon \right]
\end{equation}
Then, the $i^{\text{th}}$ histogram's bin endpoints, $\vec{b}_i$, is given by:
\begin{equation}
\label{alg:appendix:multihist}
      \vec{b}_i = \vec{b}_0 + \frac{i}{M} \mathbb{E}(w)
\end{equation}


\subsection{Proofs for Loss Gradient Norm Bounds}

\subsubsection{Distributional Loss Function}
\label{sec:loss-grad-norm-bound:distributional}

The proof for \cref{lemma:loss-grad-norm-bound:distributional} is very similar to the proof found in Appendix A of \citetappendix{imani2018improving}.

\begin{lemma}
      \label{lemma:loss-grad-norm-bound:distributional}
      The gradient norm of the histogram loss term of \methodacronym{} is bounded as follows:
      \begin{equation}
            \left\| \frac{\partial}{\partial \theta}
            {\left[\vec{p} \cdot \log{f(x)}\right]} 
            \right\|
            \le
            \left\|f(x) - \vec{p} \right\|
            \cdot \left\| {\frac{\partial g(x)}{\partial \theta}} \right\|
      \end{equation}
      where \(\frac{\partial g(x)}{\partial \theta}\) denotes the Jacobian of the model, \(g(x)\).

\end{lemma}
\begin{proof}
    We begin by taking the gradient using the chain rule.
      \begin{align*}
            \frac{\partial}{\partial \theta} {\left[\vec{p} \cdot \log{f(x)}\right]} 
             & = 
            {\frac{\partial g(x)}{\partial \theta}}
            {\frac{\partial f(x)}{\partial g(x)}}
            {\left[ \vec{p} \cdot \frac{1}{f(x)} \right]}
            \\
      {\left( {\frac{\partial f(x)}{\partial g(x)}}{\left[ \vec{p} \cdot \frac{1}{f(x)} \right]} \right)}_i
            & = \sum_j p_j \frac{1}{f_j(x)} f_j(x) \left( 1_{i=j} - f_i(x) \right)
            \\
            & = \sum_j p_j \left( 1_{i=j} - f_i(x) \right)
            \\
            & = p_i - f_i(x) \sum_j p_j
            \\
            & = p_i - f_i(x)
      \end{align*}

    We put this term back into the main equation and take the norm.
      \begin{align*}
            \frac{\partial}{\partial \theta} {\left[\vec{p} \cdot \log{f(x)}\right]} 
             & = 
            {\frac{\partial g(x)}{\partial \theta}}
            {\left[\vec{p} - f(x)\right]}
            \\
            \left\| \frac{\partial}{\partial \theta} {\left[\vec{p} \cdot \log{f(x)}\right]}  \right\| 
            & \le
            \left\|{\left[\vec{p} - f(x)\right]}\right\|
            \left\|{\frac{\partial g(x)}{\partial \theta}}\right\|
\end{align*}
    
\end{proof}

\subsubsection{Distance-based Loss Function}
\label{sec:loss-grad-norm-bound:distance-based}

\begin{lemma}
      \label{lemma:loss-grad-norm-bound:distance-based}
      The gradient norm of the distance-based loss term of \methodacronym{} is bounded as follows:
      \begin{equation}
            \left\| \frac{\partial}{\partial \theta}
            {\left| \vec{p} \cdot \vec{b} - f(x) \cdot \vec{b} ~ \right|} \right\|
            \le
            \sqrt{2} {\left\|f(x)\right\|}_2
            \cdot \left\| \vec{b} \right\|
            \cdot \left\|f(x) - \vec{p} \right\|
            \cdot \left\| {\frac{\partial g(x)}{\partial \theta}} \right\|
      \end{equation}
      where \(\frac{\partial g(x)}{\partial \theta}\) denotes the Jacobian of the model, \(g(x)\).
\end{lemma}

\begin{proof}
    We begin by taking the gradient using the chain rule.
      \begin{align*}
            \frac{\partial}{\partial \theta} \left| \vec{p} \cdot \vec{b} - f(x) \cdot \vec{b} ~ \right|
             & =
            {\frac{\partial g(x)}{\partial \theta}}
            {\frac{\partial f(x)}{\partial g(x)}}
            \left[{\frac{\partial \ell_1}{\partial f(x)}} ~ \left(\vec{p} \cdot \vec{b} - f(x) \cdot \vec{b}\right) \right] \\
            \text{where } {\left(\frac{\partial \ell_1}{\partial f(x)}\right)}_i
             & =
            \begin{cases}
                  +1,\quad {f(x)}_i > {\vec{p}}_i \\
                  -1,\quad {f(x)}_i < {\vec{p}}_i
            \end{cases}
      \end{align*}



      We then take the norm of the gradient.

      \begin{align*}
            \left\| \frac{\partial}{\partial \theta}
            {\left| \vec{p} \cdot \vec{b} - f(x) \cdot \vec{b} ~ \right|} \right\|
             & =
            \left\|
            {\frac{\partial g(x)}{\partial \theta}}
            {\frac{\partial f(x)}{\partial g(x)}}
            {\frac{\partial \ell_1}{\partial f(x)}}
            \left(\vec{p} \cdot \vec{b} - f(x) \cdot \vec{b}\right)
            \right\|
            \\
             & \le
            \left\| {\frac{\partial g(x)}{\partial \theta}} \right\|
            \left\| {\frac{\partial f(x)}{\partial g(x)}} \right\|
            \left\|\vec{p} \cdot \vec{b} - f(x) \cdot \vec{b}\right\|
      \end{align*}

      Now, \(\left\| {\frac{\partial f(x)}{\partial g(x)}} \right\|\) is the norm of the gradient of our softmax output head. We bound this norm.
      \begin{align*}
            \left\| {\frac{\partial f(x)}{\partial g(x)}} \right\|
             & \le
            \left[
                  \sum_i{
                        \sum_j{
                              \left| \frac{\partial f_i(x)}{\partial g_j(x)} \right|^2
                        }}
                  \right]^{\frac{1}{2}}
            \\
            \sum_i{
                  \sum_j{
                        \left| \frac{\partial f_i(x)}{\partial g_j(x)} \right|^2
                  }}
             & =
            \sum_i{
                  \sum_j{
                        \left| f_j(x) \left(1_{i=j} - f_i(x)\right) \right|^2
                  }}
            \\
             & =
            \left[
                  \sum_i{\left|f_i(x) \cdot \left( 1 - f_i(x) \right)\right|^2}
                  \right] \textit{~~~ when ~ \(\left(i=j\right)\)}
            \\
             & + \left[
                  \sum_i{
                        \sum_j{
                              \left|-f_i(x) ~ f_j(x)\right|^2
                        }
                  }
                  \right] \textit{~~~ when ~ \(\left(i = j\right)\) or \(\left(i \neq j\right)\)} 
            \\
             & - \left[
                  \sum_i{\left|-f_i(x) ~ f_i(x)\right|^2}
                  \right] \textit{~~~ when ~\(\left(i=j\right)\)} 
      \end{align*}

      We begin by taking the first term: \(\sum_i{\left|f_i(x) \cdot \left( 1 - f_i(x) \right)\right|^2}\).
      \begin{align*}
            \sum_i{\left|f_i(x) \cdot \left( 1 - f_i(x) \right)\right|^2}
             & = \sum_i{{\left|f_i(x) - f_i^2(x)\right|}^2} 
            \\
             & = \sum_i{{\left(f_i(x) - f_i^2(x)\right)}^2}
            \\
             & = \sum_i{f_i^2(x) ~ (f_i(x) - 1)^2}
            \\
             & \le
            \sum_i{f_i^2(x)}
      \end{align*}

      We then take the second term: \(\sum_i{
                        \sum_j{
                              \left|-f_i(x) ~ f_j(x)\right|^2
                        }
                  }\).
      \begin{align*}
            \sum_i{
                  \sum_j{
                        \left(-f_i(x) ~ f_j(x)\right)^2
                  }
            } & =
            \sum_i{
                  \sum_j{
                        f_i^2(x) ~ f_j^2(x)
                  }
            }
            \\
              & =
            \left(\sum_i{f_i^2(x)}\right) ~ \left(\sum_j{f_j^2(x)}\right)
            \\
              & = {\left(\sum_i{f_i^2(x)}\right)}^2
            \\
              & \le
            \sum_i{f_i^2(x)}
      \end{align*}

      Finally, we take the third term: \(\sum_i{\left|-f_i(x) ~ f_i(x)\right|^2}\).
      \begin{align*}
            \sum_i{\left|-f_i(x) ~ f_i(x)\right|^2}
             & = \sum_i{\left(-f_i(x) ~ f_i(x)\right)^2} \\
             & = \sum_i{\left(f_i^2(x)\right)^2} \\
             & = \sum_i{f_i^4(x)} \\
            \ge 0
      \end{align*}

      We sum all these terms.
      \begin{align*}
            \sum_i{
                  \sum_j{
                        \left| \frac{\partial f_i(x)}{\partial g_j(x)} \right|^2
                  }}
             & =
            \left[
                  \sum_i{\left|f_i(x) \cdot \left( 1 - f_i(x) \right)\right|^2}
                  \right]
            \\
             & + \left[
                  \sum_i{
                        \sum_j{
                              \left|-f_i(x) ~ f_j(x)\right|^2
                        }
                  }
                  \right] 
            \\
             & - \left[
                  \sum_i{\left|-f_i(x) ~ f_i(x)\right|^2}
                  \right] 
            \\
             & \le
            \left[\sum_i{f_i^2(x)}\right] + \left[\sum_i{f_i^2(x)}\right]
            \\
             & = 2 \sum_i{f_i^2(x)}
      \end{align*}

    We solve for the bound of \(\left\| {\frac{\partial f(x)}{\partial g(x)}} \right\|\).
      \begin{align*}
            {\left[\sum_i{
                              \sum_j{
                                    \left| \frac{\partial f_i(x)}{\partial g_j(x)} \right|^2
                              }}\right]}^{\frac{1}{2}}
             & \le {\left[2 \sum_i{f_i^2(x)}\right]}^{\frac{1}{2}}
            \\
             & = \sqrt{2} {\left\|f(x)\right\|}_2
      \end{align*}

      Finally, we evaluate for the final bounds.
      \begin{align*}
            \left\| \frac{\partial}{\partial \theta}
            {\left| \vec{p} \cdot \vec{b} - f(x) \cdot \vec{b} ~ \right|} \right\|
             & \le
            \left\| {\frac{\partial g(x)}{\partial \theta}} \right\|
            \left\| {\frac{\partial f(x)}{\partial g(x)}} \right\|
            \left\|\vec{p} \cdot \vec{b} - f(x) \cdot \vec{b}\right\|
            \\
             & \le
            \sqrt{2} {\left\|f(x)\right\|}_2
            \cdot \left\|\vec{p} \cdot \vec{b} - f(x) \cdot \vec{b}\right\|
            \cdot \left\| {\frac{\partial g(x)}{\partial \theta}} \right\|
            \\
             & =
            \sqrt{2} {\left\|f(x)\right\|}_2
            \cdot \left\| \vec{b} \right\|
            \cdot \left\|\vec{p} - f(x) \right\|
            \cdot \left\| {\frac{\partial g(x)}{\partial \theta}} \right\|
      \end{align*}
\end{proof}

\subsubsection{Gradient Norm Theorem}
\label{appendix:stablegrad:proof}
\begin{theorem}
      Assume that \(g(x)\) is locally \(l\)-Lipschitz continuous w.r.t the model's parameters, \(\theta\):
      \begin{equation}
            \left\| \frac{\partial g(x)}{\partial \theta} \right\| \leq l
      \end{equation}
      Then, the norm of the gradient of \methodacronym{} loss w.r.t. \(\theta\) is bounded by:
      \begin{equation}
            \left\| \nabla_\theta L_{\methodacronym{}}(f(x), y) \right\|
            \leq
            l~\left\|\vec{p} - f(x)\right\|
            \left[
            1+
            \sqrt{2} ~
            {\left\|f(x)\right\|}_2
            \left\|\vec{b}\right\|
            \right]
      \end{equation}
      where $\vec{p}$ is the histogram representation of the target (\ie, $\vec{p} = Y = \Phi(y)$).
\end{theorem}

\begin{proof}
      The gradient of the \methodacronym{} loss w.r.t. \(\theta\) is:
      \begin{equation}
            \nabla_\theta L_{\methodacronym{}}(f(x), y)
            = \frac{\partial}{\partial \theta}\left[ \vec{p} \cdot \log{f(x)} \right] + \frac{\partial}{\partial \theta}  \left| \vec{p} \cdot \vec{b} - f(x) \cdot \vec{b} \right|
      \end{equation}
      where the first term is the gradient of the histogram component of the loss w.r.t. \(\theta\) and the second term is the gradient of the distance-based component of the loss w.r.t. \(\theta\).

      For the histogram term, we utilize \cref{lemma:loss-grad-norm-bound:distributional}, which provides the following bounds:
      \begin{equation*}
            \left\| \frac{\partial}{\partial \theta}
            {\left( \vec{p} \cdot \log{f(x)} \right)} \right\| \le
            \left\| {\frac{\partial g(x)}{\partial \theta}} \right\|
            \cdot \left\| \vec{p} - f(x) \right\|
      \end{equation*}

      For the distance-based term, we utilize \cref{lemma:loss-grad-norm-bound:distance-based}, which provides the following bounds:
      \begin{equation*}
            \left\| \frac{\partial}{\partial \theta}
            {\left| \vec{p} \cdot \vec{b} - f(x) \cdot \vec{b} ~ \right|} \right\| \le
            \left\| {\frac{\partial g(x)}{\partial \theta}} \right\|
            \cdot \left[
                  \sqrt{2}
                  \cdot {\left\|f(x)\right\|}_2
                  \cdot \left\| \vec{b} \right\|
                  \cdot \left\|\vec{p} - f(x) \right\|
                  \right]
      \end{equation*}

      Putting the two terms together, we obtain the following bound:
      \begin{equation}
            \left\| \nabla_\theta L_{\methodacronym{}}(f(x), y) \right\|
            \leq
            l~\left\|\vec{p} - f(x)\right\|
            \left[
            1+
            \sqrt{2} ~
            {\left\|f(x)\right\|}_2
            \left\|\vec{b}\right\|
            \right]
      \end{equation}
\end{proof}

\subsection{\({\left\|f^2(x)\right\|}_2\) Values for Different Distributions}
\Cref{fig:appendix_distributions_l2} shows the \({\left\|f^2(x)\right\|}_2\) values for the normal, categorical, and uniform distributions.

\subsection{Dataset Histograms}
\label{sec:appendix:dataset_histograms}

\subsubsection{OC20 Dataset}

\Cref{fig:appendix:data-hist:oc20} shows KDE plots for the data distributions of all different splits of the OC20 validation dataset.

\begin{figure*}[h]
    \includegraphics[width=1.0\textwidth]{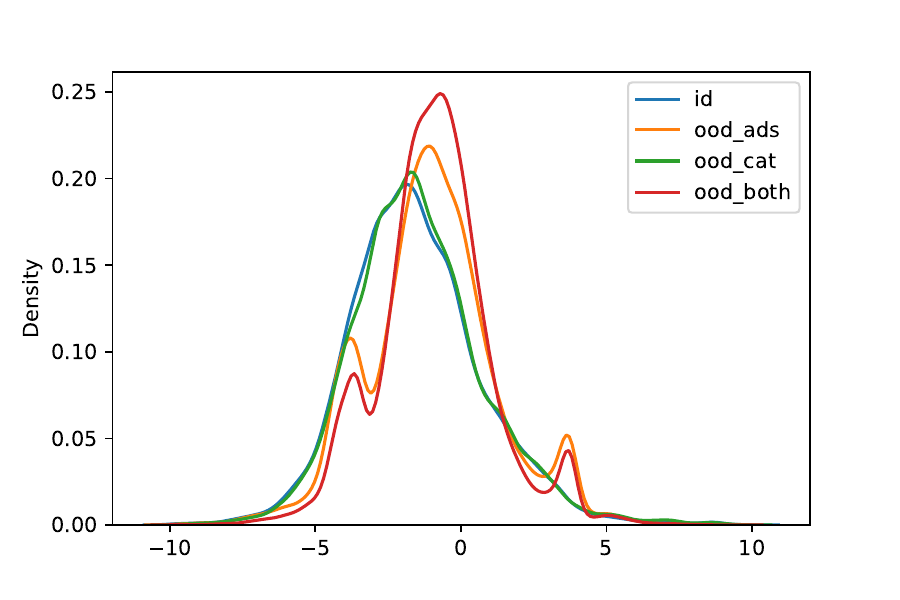}
    \caption{KDE plots for the data distribution of the ID, OOD Ads, OOD Cat, and OOD Both splits of the OC20 validation dataset.}
    \label{fig:appendix:data-hist:oc20}
\end{figure*}

\subsubsection{MD17 Dataset}
\Cref{fig:appendix:data-hist:md17:aspirin,fig:appendix:data-hist:md17:benzene_old,fig:appendix:data-hist:md17:ethanol,fig:appendix:data-hist:md17:malonaldehyde,fig:appendix:data-hist:md17:naphthalene,fig:appendix:data-hist:md17:salicylic,fig:appendix:data-hist:md17:toluene,fig:appendix:data-hist:md17:uracil} show KDE plots for the data distributions of all different molecules in the validation split of the MD17 dataset.

\begin{figure*}[h]
    \includegraphics[width=1.0\textwidth]{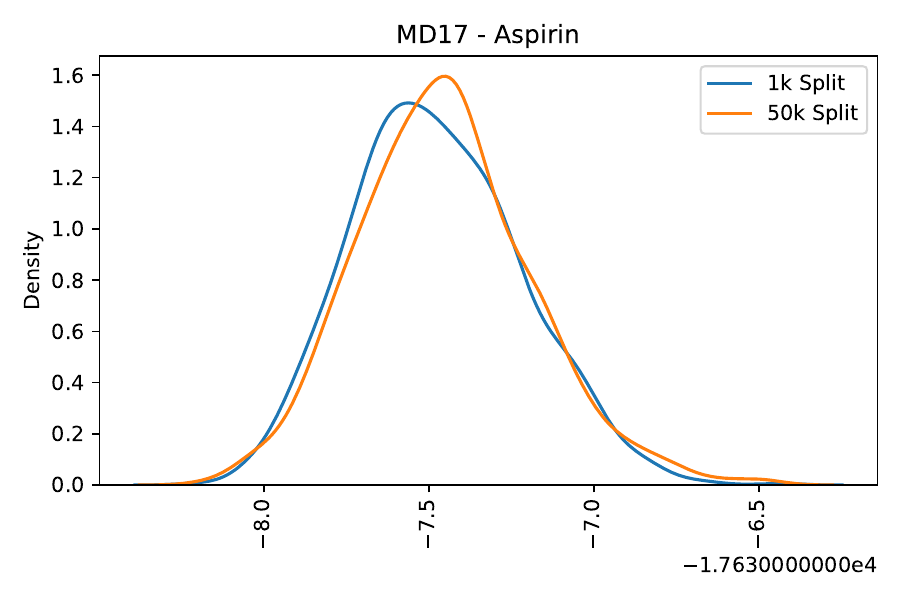}
    \caption{KDE plots for the data distribution of the Aspirin molecule in the MD17 Dataset}
    \label{fig:appendix:data-hist:md17:aspirin}
\end{figure*}
\begin{figure*}[h]
    \includegraphics[width=1.0\textwidth]{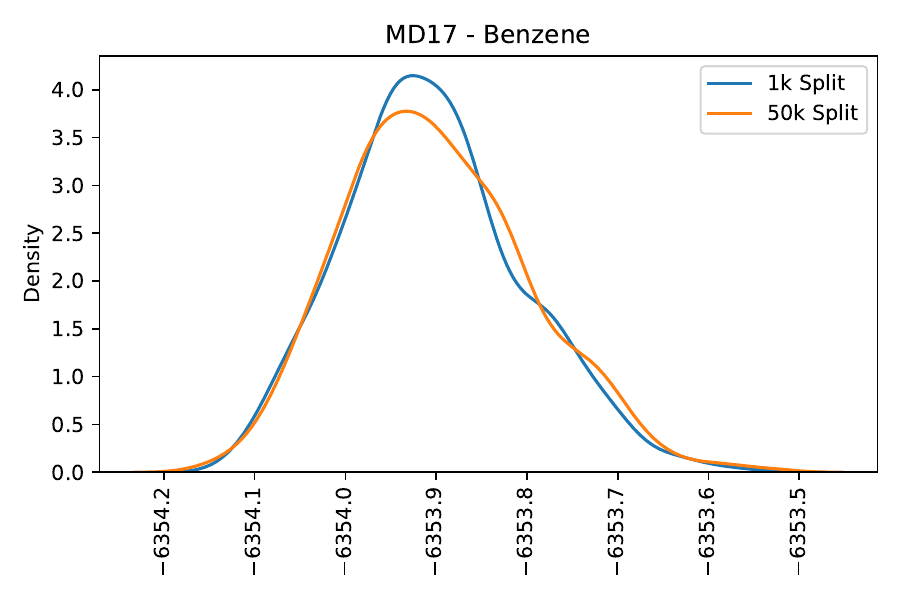}
    \caption{KDE plots for the data distribution of the Benzene molecule in the MD17 Dataset}
    \label{fig:appendix:data-hist:md17:benzene_old}
\end{figure*}
\begin{figure*}[h]
    \includegraphics[width=1.0\textwidth]{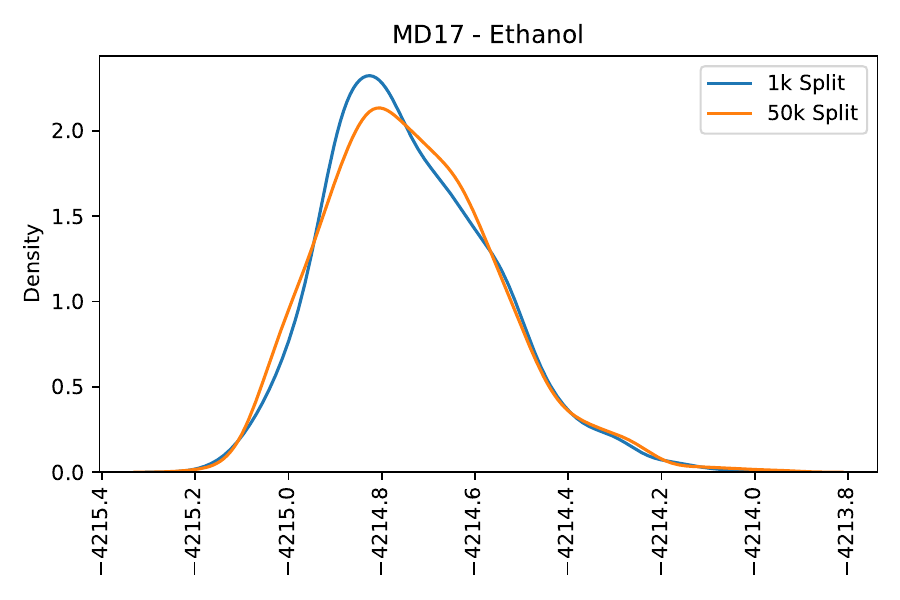}
    \caption{KDE plots for the data distribution of the Ethanol molecule in the MD17 Dataset}
    \label{fig:appendix:data-hist:md17:ethanol}
\end{figure*}
\begin{figure*}[h]
    \includegraphics[width=1.0\textwidth]{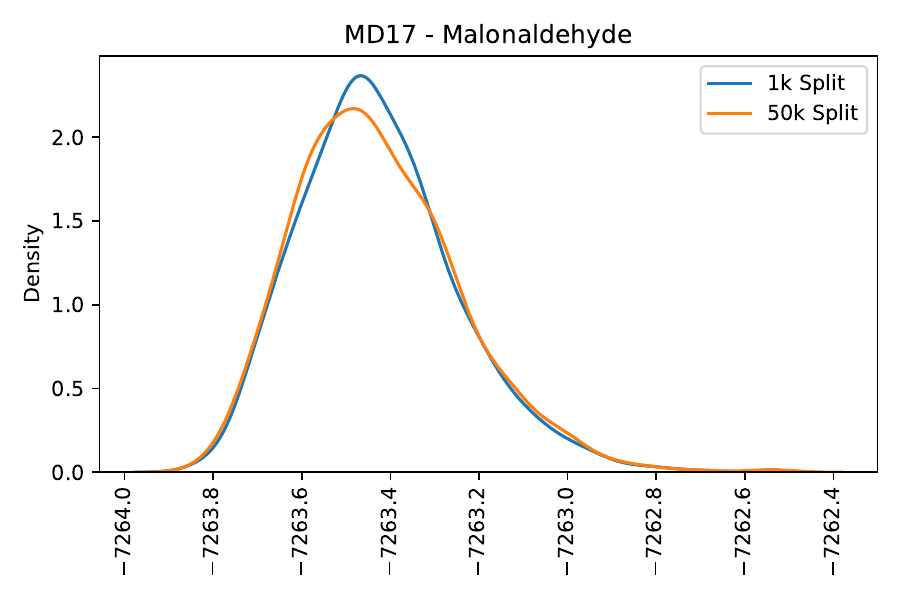}
    \caption{KDE plots for the data distribution of the Malonaldehyde molecule in the MD17 Dataset}
    \label{fig:appendix:data-hist:md17:malonaldehyde}
\end{figure*}
\begin{figure*}[h]
    \includegraphics[width=1.0\textwidth]{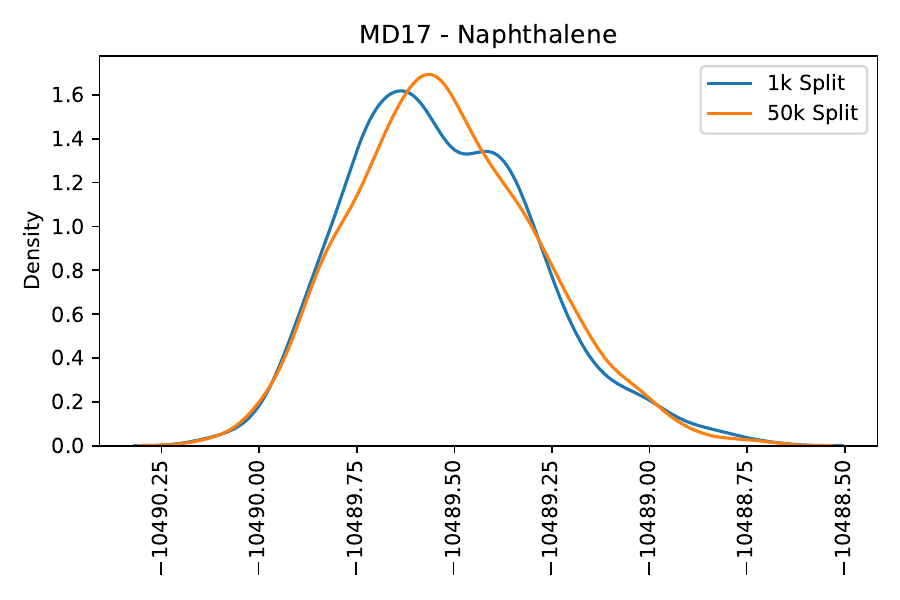}
    \caption{KDE plots for the data distribution of the Naphthalene molecule in the MD17 Dataset}
    \label{fig:appendix:data-hist:md17:naphthalene}
\end{figure*}
\begin{figure*}[h]
    \includegraphics[width=1.0\textwidth]{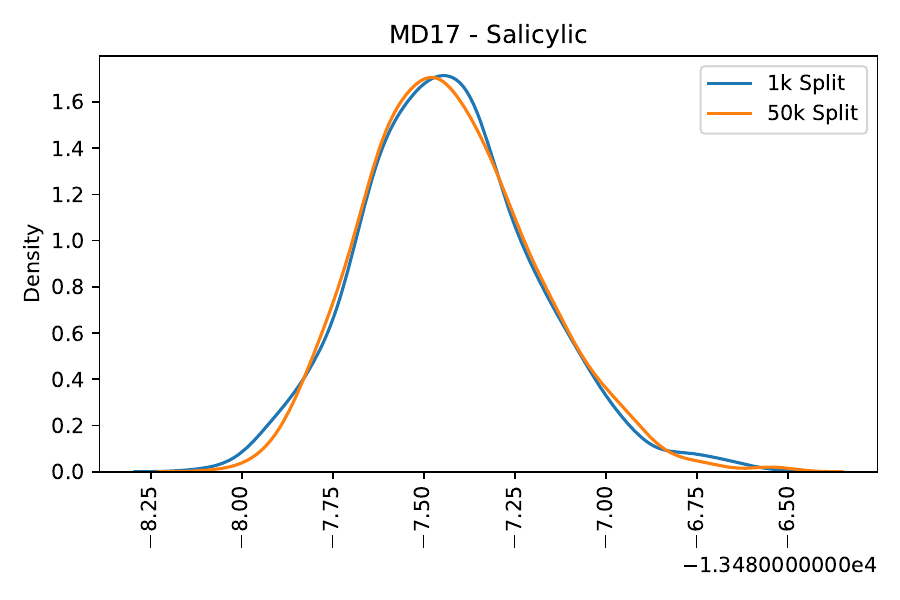}
    \caption{KDE plots for the data distribution of the Salicylic molecule in the MD17 Dataset}
    \label{fig:appendix:data-hist:md17:salicylic}
\end{figure*}
\begin{figure*}[h]
    \includegraphics[width=1.0\textwidth]{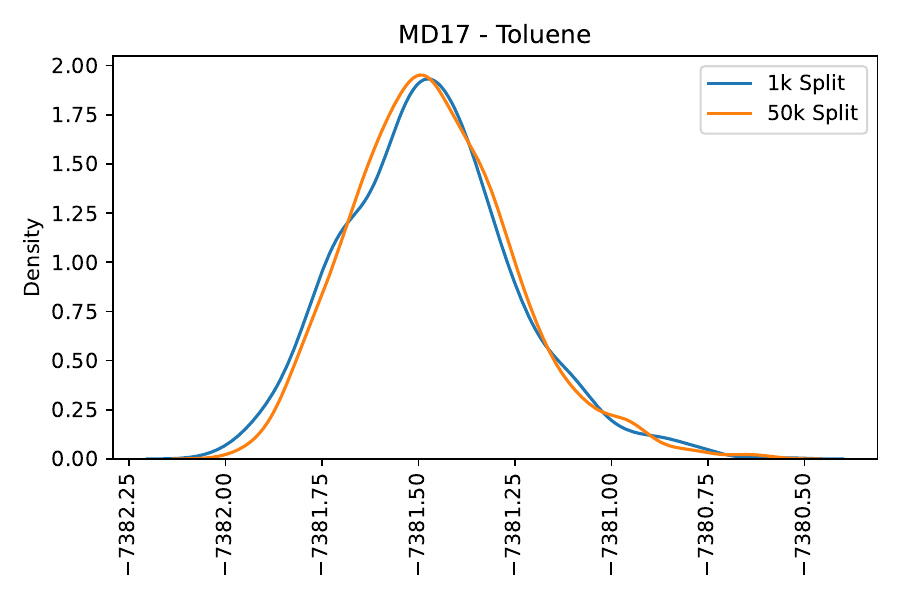}
    \caption{KDE plots for the data distribution of the Toluene molecule in the MD17 Dataset}
    \label{fig:appendix:data-hist:md17:toluene}
\end{figure*}
\begin{figure*}[h]
    \includegraphics[width=1.0\textwidth]{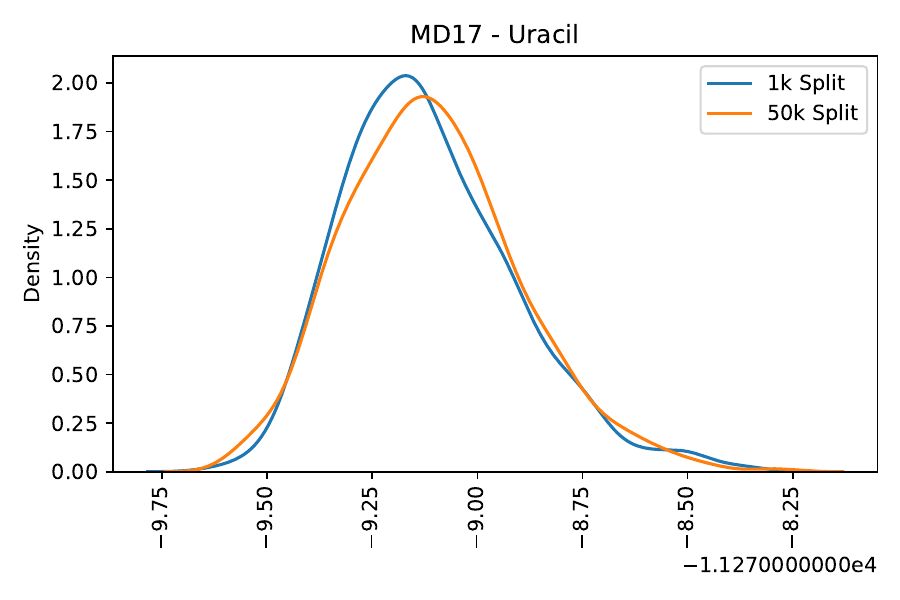}
    \caption{KDE plots for the data distribution of the Uracil molecule in the MD17 Dataset}
    \label{fig:appendix:data-hist:md17:uracil}
\end{figure*}

\subsubsection{QM9 Dataset}
\Cref{fig:appendix:data-hist:qm9:0,fig:appendix:data-hist:qm9:1,fig:appendix:data-hist:qm9:2,fig:appendix:data-hist:qm9:3,fig:appendix:data-hist:qm9:4,fig:appendix:data-hist:qm9:5,fig:appendix:data-hist:qm9:6,fig:appendix:data-hist:qm9:7,fig:appendix:data-hist:qm9:8,fig:appendix:data-hist:qm9:9,fig:appendix:data-hist:qm9:10,fig:appendix:data-hist:qm9:11} show KDE plots for the data distributions of all different targets in the validation split of the QM9 dataset.

\begin{figure*}[h]
    \includegraphics[width=1.0\textwidth]{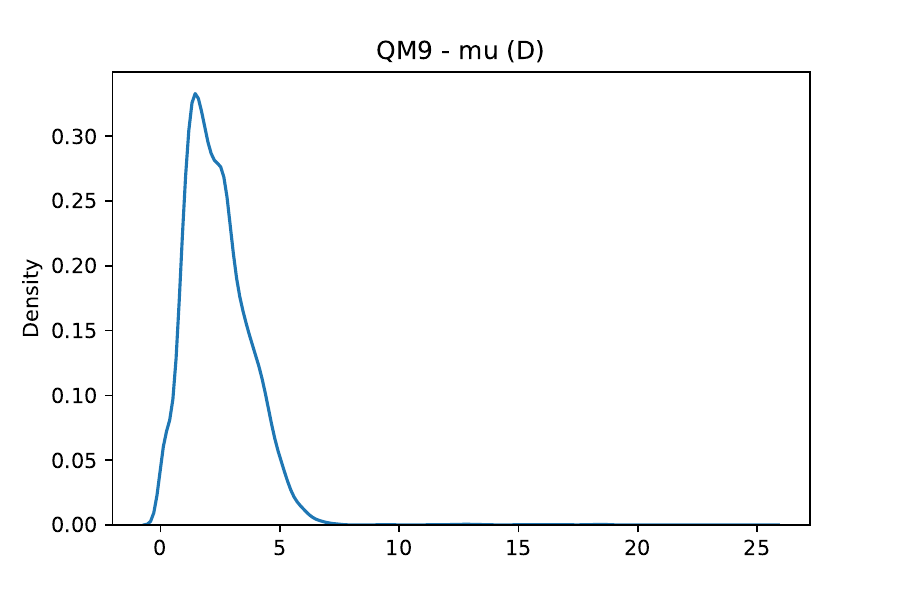}
    \caption{KDE plots for the data distribution of the $mu$ ($D$) target in the QM9 Dataset}
    \label{fig:appendix:data-hist:qm9:0}
\end{figure*}
\begin{figure*}[h]
    \includegraphics[width=1.0\textwidth]{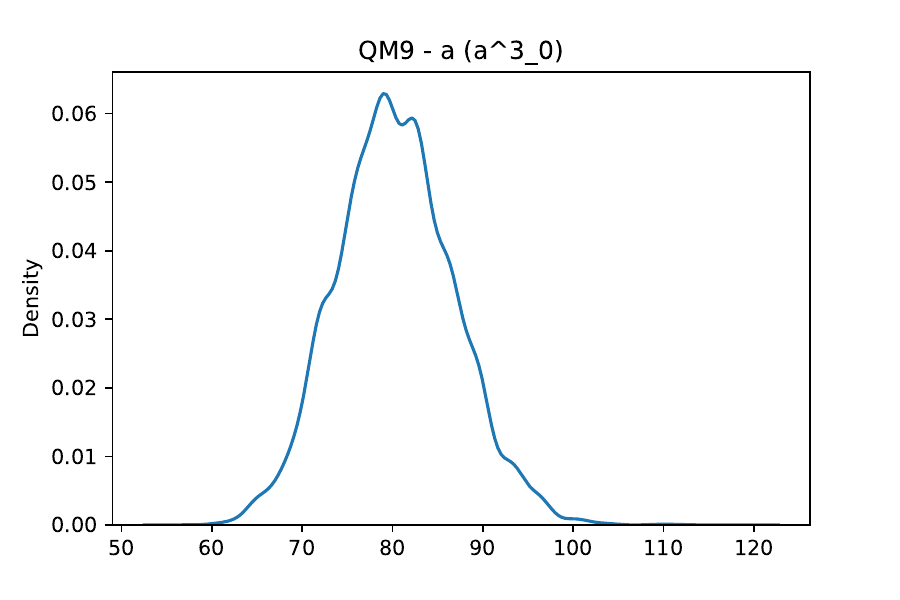}
    \caption{KDE plots for the data distribution of the $a$ ($a^3_0$) target in the QM9 Dataset}
    \label{fig:appendix:data-hist:qm9:1}
\end{figure*}
\begin{figure*}[h]
    \includegraphics[width=1.0\textwidth]{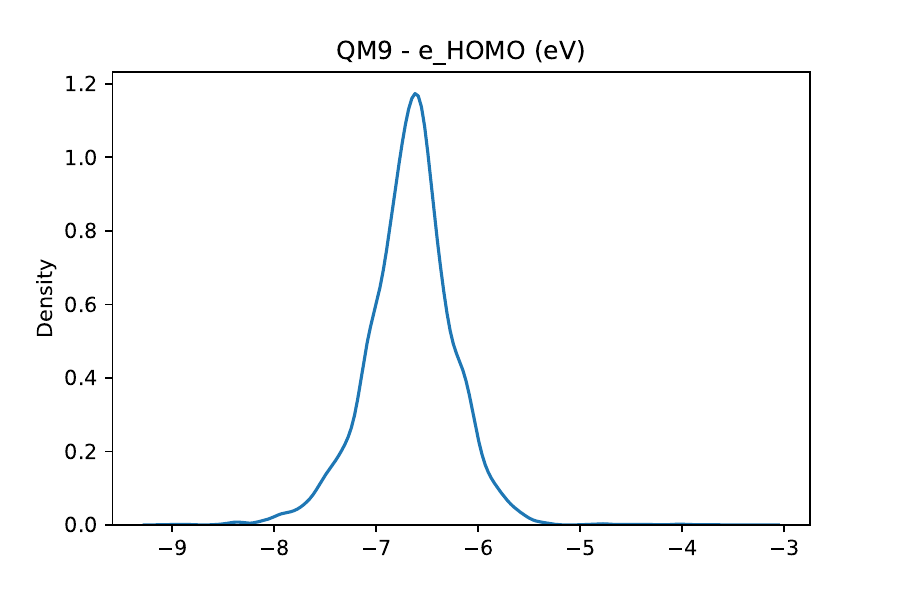}
    \caption{KDE plots for the data distribution of the $e_{HOMO}$ ($eV$) target in the QM9 Dataset}
    \label{fig:appendix:data-hist:qm9:2}
\end{figure*}
\begin{figure*}[h]
    \includegraphics[width=1.0\textwidth]{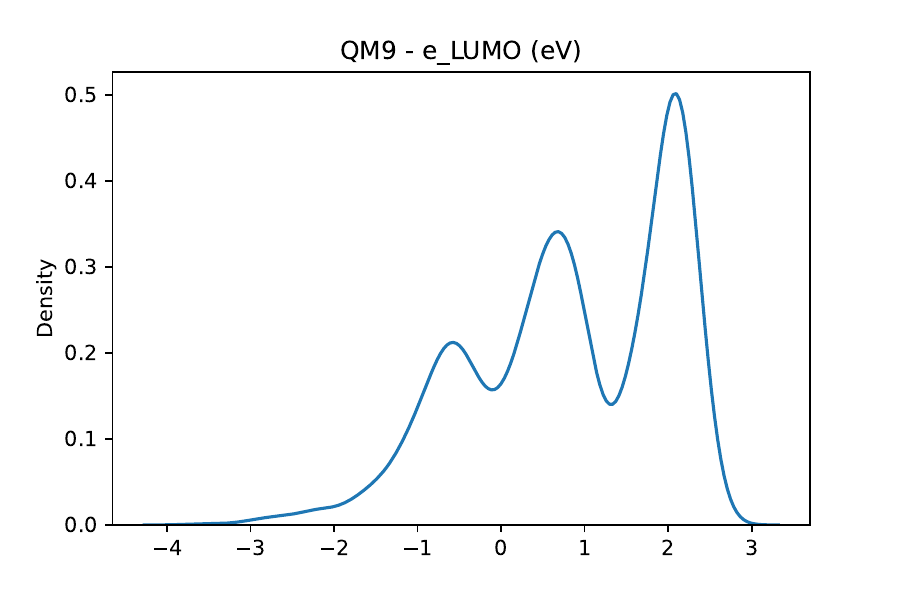}
    \caption{KDE plots for the data distribution of the $e_{LUMO}$ ($eV$) target in the QM9 Dataset}
    \label{fig:appendix:data-hist:qm9:3}
\end{figure*}
\begin{figure*}[h]
    \includegraphics[width=1.0\textwidth]{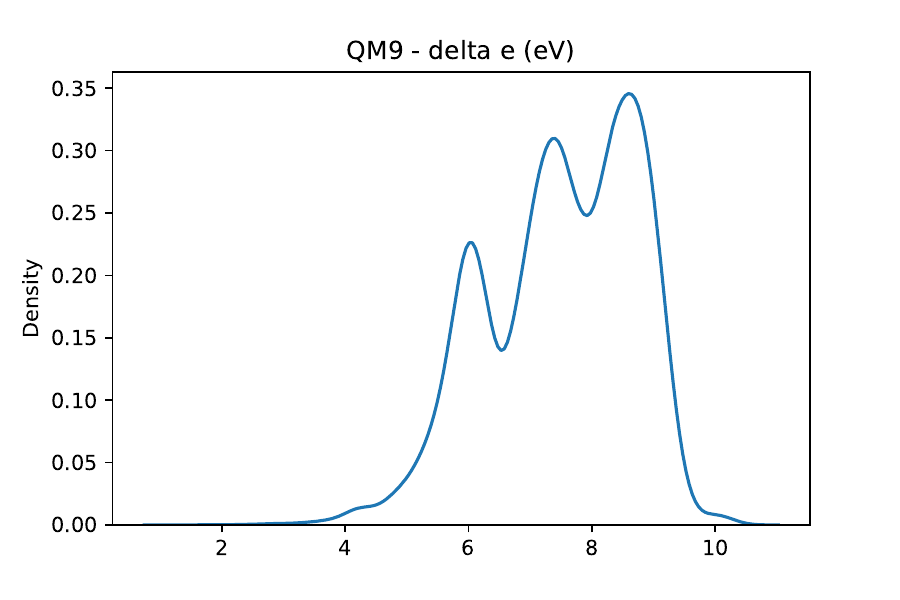}
    \caption{KDE plots for the data distribution of the $\Delta e$ ($eV$) target in the QM9 Dataset}
    \label{fig:appendix:data-hist:qm9:4}
\end{figure*}
\begin{figure*}[h]
    \includegraphics[width=1.0\textwidth]{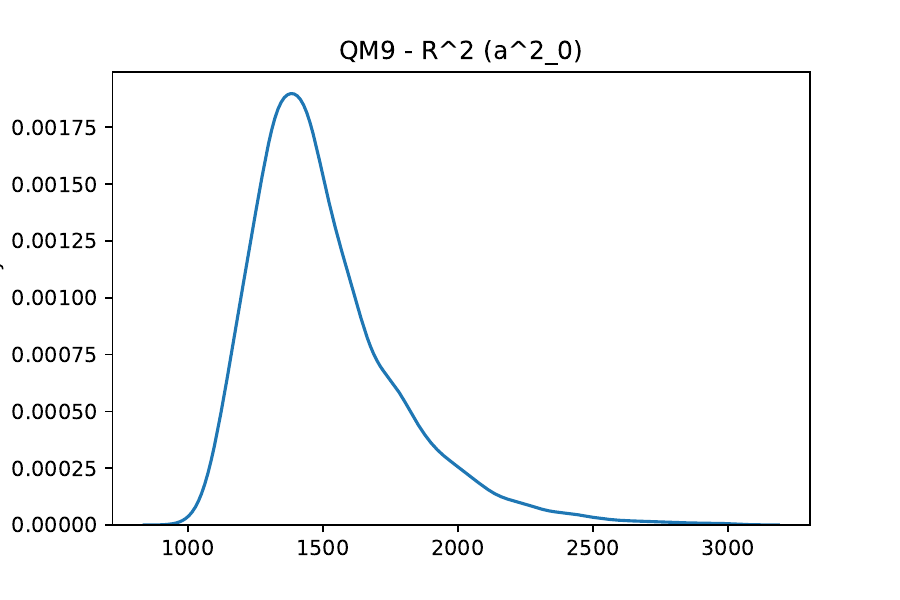}
    \caption{KDE plots for the data distribution of the $R^2$ ($a^2_0$) target in the QM9 Dataset}
    \label{fig:appendix:data-hist:qm9:5}
\end{figure*}
\begin{figure*}[h]
    \includegraphics[width=1.0\textwidth]{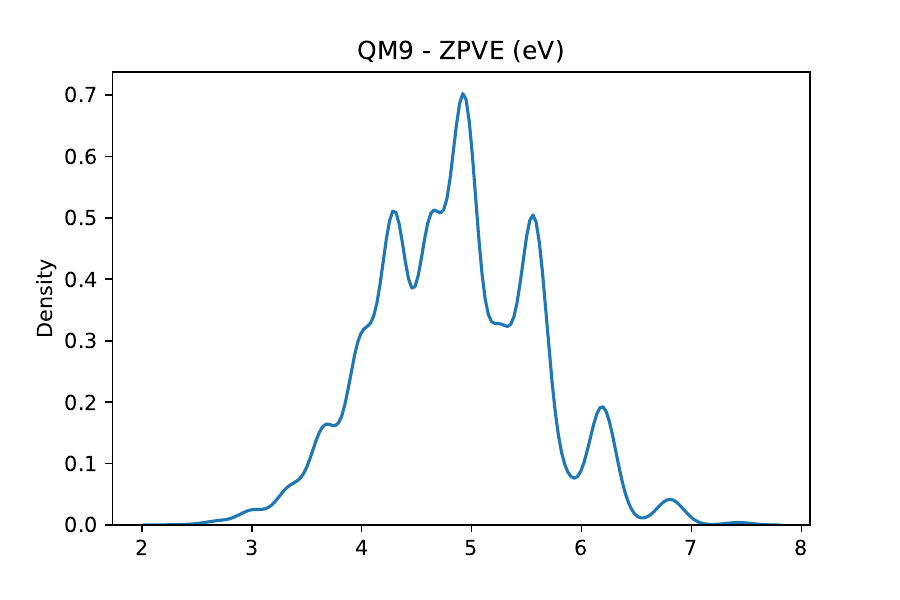}
    \caption{KDE plots for the data distribution of the $ZPVE$ ($eV$) target in the QM9 Dataset}
    \label{fig:appendix:data-hist:qm9:6}
\end{figure*}
\begin{figure*}[h]
    \includegraphics[width=1.0\textwidth]{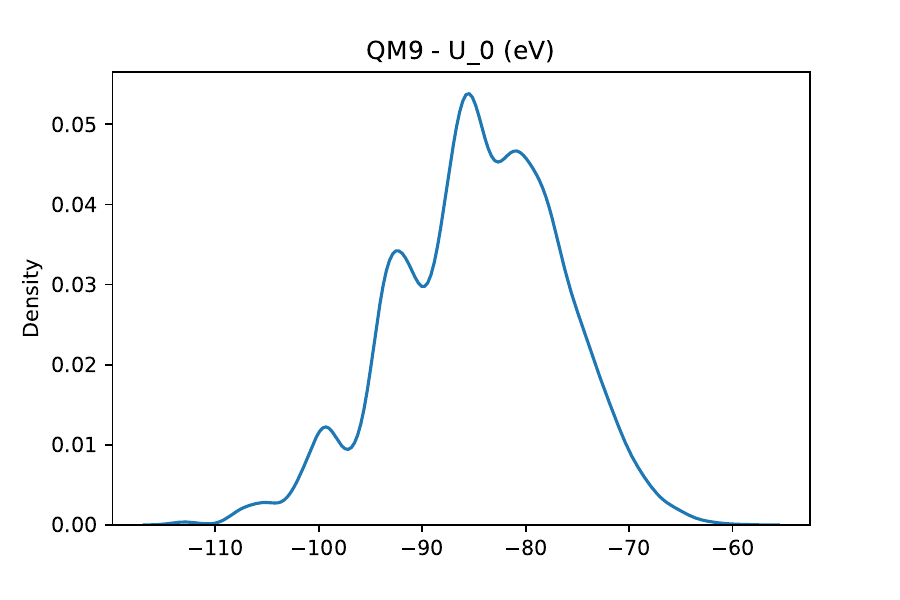}
    \caption{KDE plots for the data distribution of the $U_0$ ($eV$) target in the QM9 Dataset}
    \label{fig:appendix:data-hist:qm9:7}
\end{figure*}
\begin{figure*}[h]
    \includegraphics[width=1.0\textwidth]{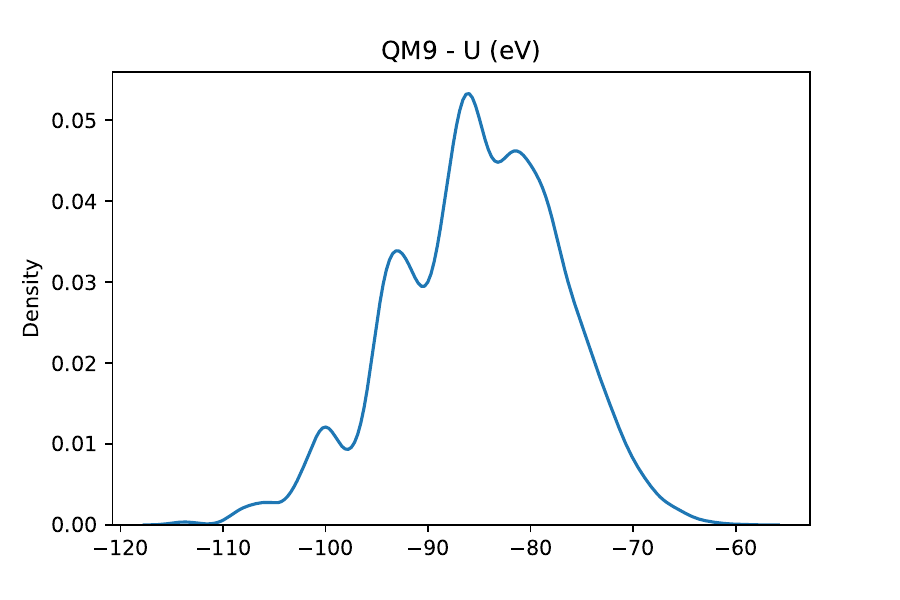}
    \caption{KDE plots for the data distribution of the $U$ ($eV$) target in the QM9 Dataset}
    \label{fig:appendix:data-hist:qm9:8}
\end{figure*}
\begin{figure*}[h]
    \includegraphics[width=1.0\textwidth]{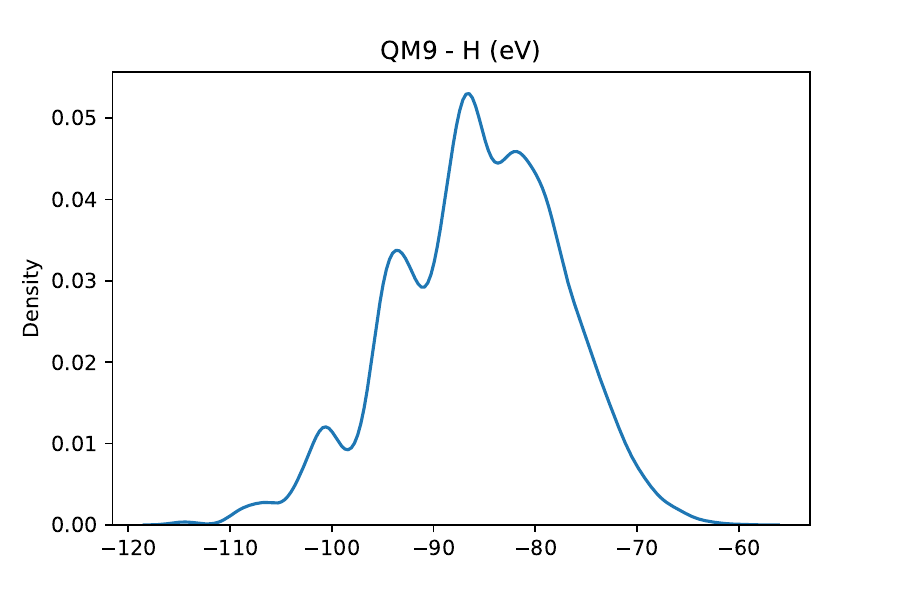}
    \caption{KDE plots for the data distribution of the $H$ ($eV$) target in the QM9 Dataset}
    \label{fig:appendix:data-hist:qm9:9}
\end{figure*}
\begin{figure*}[h]
    \includegraphics[width=1.0\textwidth]{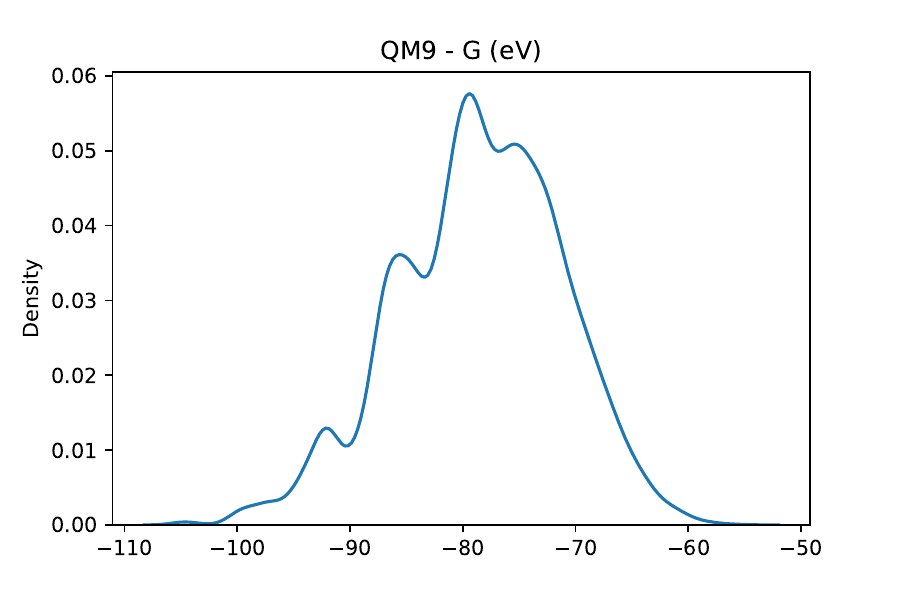}
    \caption{KDE plots for the data distribution of the $G$ ($eV$) target in the QM9 Dataset}
    \label{fig:appendix:data-hist:qm9:10}
\end{figure*}
\begin{figure*}[h]
    \includegraphics[width=1.0\textwidth]{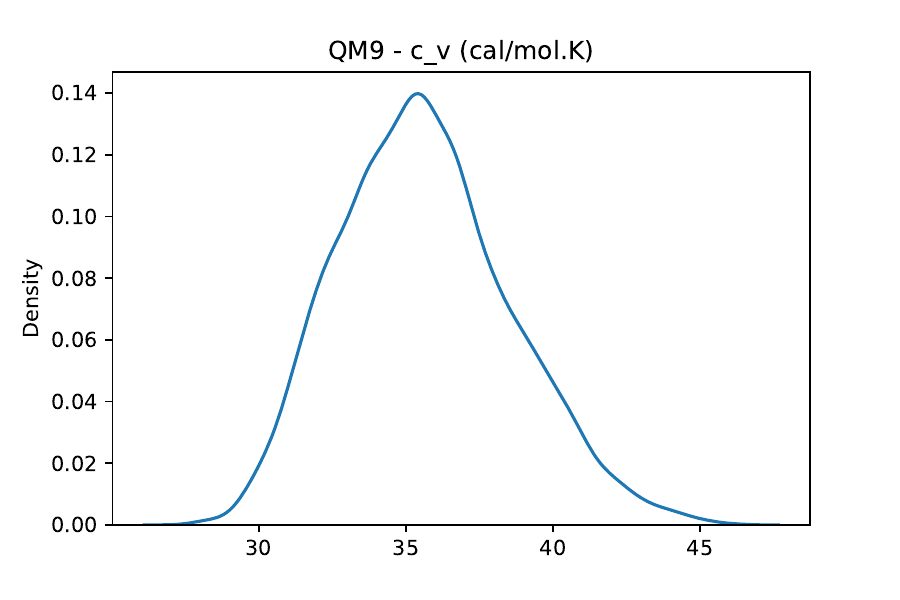}
    \caption{KDE plots for the data distribution of the $c_v$ (cal/mol.K)) target in the QM9 Dataset}
    \label{fig:appendix:data-hist:qm9:11}
\end{figure*}

\subsection{Additional Ablations}

\subsubsection{Induced Distribution}
\label{sec:induced-distribution-ablation}
\Cref{tab:appendix:induced-dist} shows ID validation metrics for training the GemNet+Pos+DMoE model, with different induced distributions, on the OC20 dataset. Specifically, we experiment with Normal, Laplace, Categorical (\ie, one-hot), and K-Categorical (\ie normalized k-hot) distributions. For the normal distribution, we experiment with different values for $\sigma$.

Our results show that the exact details of the choice of induce distribution does not seem to have a huge impact on the accuracy results and much of the differences in performance fall within the margin of error. However, there are some very clear trends that we will analyze below. All these trends are independent of our choice of bin distribution.

The normal distribution (originally suggested by \citetappendix{imani2018improving}) is the optimal choice for the induced distribution. 
The categorical distribution (referred to as the dirac delta distribution in \citetappendix{imani2018improving}) demonstrates the worst performance. 
For the normal distribution, lower $\sigma$ values yield, on average, higher EWT scores but also higher MAE results. Overall, our empirical results demonstrate that using a normal induced distribution with $\sigma=1\cdot\mathbb{E}{\left[w\right]}$ provides a good balance between MAE and EWT results.

\begin{table}[!htp]\centering
\caption{ID Validation metrics for GemNet+Pos+DMoE models trained with different induced distributions on the OC20 dataset}\label{tab:appendix:induced-dist}
\scriptsize
\begin{tabular}{lrrrrrr}\toprule
& &\multicolumn{2}{c}{Normal Bin Distribution} &\multicolumn{2}{c}{Uniform Bin Distribution} \\\cmidrule{3-6}
& &Energy MAE &EWT &Energy MAE &EWT \\\midrule
\multirow{6}{*}{Normal} &$\sigma=0.5 \cdot \mathbb{E}{\left[w\right]}$ &0.458 &8.0\% &0.459 &\textbf{8.0\%} \\
&$\sigma=1 \cdot \mathbb{E}{\left[w\right]}$ &0.458 &\textbf{8.1\%} &0.457 &7.9\% \\
&$\sigma=1.5 \cdot \mathbb{E}{\left[w\right]}$ &\textbf{0.453} &7.9\% &0.459 &7.9\% \\
&$\sigma=2 \cdot \mathbb{E}{\left[w\right]}$ &0.457 &7.8\% &\textbf{0.456} &\textbf{8.0\%} \\
&$\sigma=5 \cdot \mathbb{E}{\left[w\right]}$ &0.455 &7.3\% &0.458 &7.9\% \\
&$\sigma=10 \cdot \mathbb{E}{\left[w\right]}$ &0.462 &7.0\% &\textbf{0.456} &7.6\% \\
Laplace &$b=1 \cdot \mathbb{E}{\left[w\right]}$ &0.457 &7.7\% &0.462 &7.9\% \\
K-Categorical &$k=3$ &0.459 &7.8\% &0.459 &7.8\% \\
Categorical & &0.460 &7.3\% &0.463 &7.6\% \\
\bottomrule
\end{tabular}
\end{table}

\subsubsection{Comparison to Robust Regression Losses}
\label{sec:robust-regression-losses}
\Cref{tab:appendix:robust-losses} shows ID validation metrics for training the GemNet+Pos model on the OC20 dataset, using different loss functions. Specifically, we compare our \methodacronym{} (using normal bin distributions) loss with L1 loss, L2 loss, and Smooth L1 Loss~\citeappendix{girshick2015fast} (with \(\beta=1.0\)).
These results show that the performance boost that \methodacronym{} provides is seemingly different from the effect observed from robust regression losses.

\begin{table}[!htp]\centering
    \caption{ID Validation metrics for GemNet+Pos models trained with different loss functions on the OC20 dataset}\label{tab:appendix:robust-losses}
    \scriptsize
    \begin{tabular}{lrrr}\toprule
                       & Energy         & EwT            \\\midrule
        L1 Loss        & 0.476          & 5.7\%          \\
        L2 Loss        & 0.502          & 4.3\%          \\
        Smooth L1 Loss & 0.505          & 4.4\%          \\
        DMoE           & \textbf{0.450} & \textbf{7.8\%} \\
        \bottomrule
    \end{tabular}
\end{table}

\subsection{Optimized Multi-Histogram Implementation}
\label{appendix:optimized-multihist}
In classic PyTorch code, creating a multi-histogram output head is slow because heads get evaluated sequentially. However, we use FuncTorch~\citeappendix{functorch2021} to run all these heads in parallel. An example is shown below.

Let us use the output heads and input shown below.
\begin{lstlisting}
num_histograms = 10

output_heads = [
    nn.Sequential(
        nn.Linear(input_dim, input_dim // 2),
        nn.ReLU(),
        nn.Linear(input_dim // 2, num_histograms),
    ),
    for _ in range(num_models)
]

input = torch.randn(1, input_dim)
\end{lstlisting}

In classic PyTorch, we would use the following sequential code to evaluate all heads:
\begin{lstlisting}
predictions = [output_head(input) for output_head in output_heads]
\end{lstlisting}

However, with the parallel optimization, we use the following code:
\begin{lstlisting}
fmodel, params, buffers = (
    functorch.combine_state_for_ensemble(output_heads)
)
evaluate_all = functorch.vmap(fmodel, in_dims=(0, 0, None))
predictions = evaluate_all(params, buffers, input)
\end{lstlisting}

\subsection{Model Hyperparameters}
In the section below, we enumerate all chosen hyperparameters across all our experiments. 
Comma-separated hyperparameter values indicate that multiple values were evaluated in our experiments. 
In this case, an exhaustive grid search is conducted across all possible hyperparameters, and the hyperparameters that produce the best results (\ie, the lowest validation MAE scores) are selected.

For the histogram induced distribution hyperparameter, \(w\) refers to the bin width vector; thus, \(E\left[w\right]\) refers to the average bin width in the histogram.

\subsubsection{OC20 Models}
\Cref{tab:hparams:oc20:schnet,tab:hparams:oc20:dimenet,tab:hparams:oc20:gemnet,tab:hparams:oc20:graphormer,tab:hparams:oc20:gemnet_star} show the hyperperameters for the SchNet, DimeNet++, GemNet-dT, Graphormer, and GemNet* models, respectively, as shown on \cref{tab:oc20-val,tab:oc20-test}.

\begin{table}[!htp]\centering
\caption{SchNet Hyperparameters for OC20 Dataset}\label{tab:hparams:oc20:schnet}
\scriptsize
\begin{tabular}{lrr}\toprule
Hyperparameter &Value \\\midrule
Number of Hidden Channels &384 \\
Number of Filters to Use &128 \\
Number of Interaction Blocks &4 \\
Number of Gaussians &100 \\
Cutoff Distance for Interatomic Interactions &6 \\
Optimizer &Adam \\
Learning Rate &0.001 \\
Batch Size &32 \\
Training Time &100 epochs with early stopping \\
Number of Histograms &1 \\
Number of Histogram Bins &1024 \\
Histogram Bin Distribution &Uniform \\
Histogram Induced Distribution &Normal($\sigma = \mathbb{E}\left[w\right]$) \\
\bottomrule
\end{tabular}
\end{table}
\begin{table}[!htp]\centering
\caption{DimeNet++ Hyperparameters for OC20 Dataset}\label{tab:hparams:oc20:dimenet}
\scriptsize
\begin{tabular}{lrr}\toprule
Hyperparameter &Value \\\midrule
Hidden Embedding Size &256 \\
Embedding Size Used for Atoms in the Output Block &192 \\
Number of Building Blocks &3 \\
Cutoff Distance for Interatomic Interactions &6 \\
Number of Radial Basis Functions &6 \\
Number of Spherical Harmonics &7 \\
Number of Residual Layers in the Interaction Blocks Before the Skip Connection &1 \\
Number of Residual Layers in the Interaction Blocks After the Skip Connection &2 \\
Number of Linear Layers for the Output Blocks &3 \\
Optimizer &Adam \\
Learning Rate &0.001 \\
Batch Size &12 \\
Training Time &100 epochs with early stopping \\
Number of Histograms &1 \\
Number of Histogram Bins &1024 \\
Histogram Induced Distribution &Normal($\sigma = \mathbb{E}\left[w\right]$) \\
\bottomrule
\end{tabular}
\end{table}
\begin{table}[!htp]\centering
\caption{GemNet-dT Hyperparameters for OC20 Dataset}\label{tab:hparams:oc20:gemnet}
\scriptsize
\begin{tabular}{lrr}\toprule
Hyperparameter &Value \\\midrule
Number of Radial Basis Functions &128 \\
Number of Spherical Harmonics &7 \\
Number of Building Blocks &3 \\
Number of Output Blocks &4 \\
Embedding Size of the Atoms &512 \\
Embedding Size of the Edges &512 \\
Embedding Size in the Triplet Message Passing Block &64 \\
Embedding Size of the Radial Basis Transformation &16 \\
Embedding Size of the Circular Basis Transformation &16 \\
Edge Embedding Size After the Bilinear Layer &64 \\
Cutoff Distance for Interatomic Interactions &6 \\
Number of Residual Blocks After the Concatenation &1 \\
Number of Residual Blocks in the Atom Embedding Blocks &3 \\
Number of Residual Layers Before Skip &1 \\
Number of Residual Layers After Skip &2 \\
Number of Linear Layers for the Output Blocks &3 \\
Cutoff Distance for Interatomic Interactions &6 \\
Max Number of Neighbors for Interatomic Interactions &50 \\
Auxiliary Position Loss Coefficient &16 \\
Energy Loss Coefficient &1 \\
Optimizer &AdamW \\
Learning Rate &0.0005 \\
Learning Rate Scheduler &ReduceLROnPlateau \\
Max Gradient Norm (i.e., Clip Gradient Norm Value) &10 \\
Batch Size &4, 8, 16 \\
Training Time &100 epochs with early stopping \\
Number of Histograms &1, 8, 32, 256 \\
Number of Histogram Bins &256, 1024, 2048, 4096 \\
Histogram Bin Distribution &Uniform, Normal \\
Histogram Induced Distribution &Normal($\sigma = \mathbb{E}\left[w\right]$), Laplace($b = \mathbb{E}\left[w\right]$), Categorical, K-Categorical \\
\bottomrule
\end{tabular}
\end{table}
\begin{table}[!htp]\centering
\caption{Graphormer Hyperparameters for OC20 Dataset}\label{tab:hparams:oc20:graphormer}
\scriptsize
\begin{tabular}{lrr}\toprule
Hyperparameter &Value \\\midrule
Number of Radial Basis Functions &128 \\
Number of Building Blocks &12 \\
Number of Building Block Repeats &4 \\
Embedding Size &768 \\
FFN Embedding Size &768 \\
Number of Attention Heads &48 \\
Input Dropout Rate &0 \\
Encoder Layer Dropout Rate &0.1 \\
Attention Dropout Rate &0.1 \\
Activation Dropout Rate &0 \\
Cutoff Distance for Interatomic Interactions &12 \\
Auxiliary Position Loss Coefficient &16 \\
Energy Loss Coefficient &1 \\
Optimizer &Adam \\
Learning Rate &0.0003 \\
Learning Rate Scheduler &Polynomial decay to 0 over 1 million train steps \\
Max Gradient Norm (i.e., Clip Gradient Norm Value) &5 \\
Batch Size &2, 4, 6 \\
Training Time &100 epochs with early stopping \\
Number of Histograms &1, 2, 4 \\
Number of Histogram Bins &256, 1024, 2048, 4096 \\
Histogram Bin Distribution &Uniform, Normal \\
Histogram Induced Distribution &Normal($\sigma = \mathbb{E}\left[w\right]$), Laplace($b = \mathbb{E}\left[w\right]$), Categorical, K-Categorical \\
\bottomrule
\end{tabular}
\end{table}
\begin{table}[!htp]\centering
\caption{GemNet* Hyperparameters for OC20 Dataset}\label{tab:hparams:oc20:gemnet_star}
\scriptsize
\begin{tabular}{lrr}\toprule
Hyperparameter &Value \\\midrule
Number of Radial Basis Functions &128 \\
Number of Spherical Harmonics &7 \\
Number of Building Blocks &12 \\
Number of Building Block Repeats &3 \\
Number of Output Blocks &1 \\
Embedding Size of the Atoms &512 \\
Embedding Size of the Edges &512 \\
Embedding Size in the Triplet Message Passing Block &128 \\
Embedding Size of the Radial Basis Transformation &128 \\
Embedding Size of the Circular Basis Transformation &16 \\
Edge Embedding Size After the Bilinear Layer &128 \\
Cutoff Distance for Interatomic Interactions &12 \\
Number of Residual Blocks After the Concatenation &1 \\
Number of Residual Blocks in the Atom Embedding Blocks &3 \\
Number of Residual Layers Before Skip &1 \\
Number of Residual Layers After Skip &2 \\
Number of Linear Layers for the Output Blocks &3 \\
Cutoff Distance for Interatomic Interactions &6 \\
Max Number of Neighbors for Interatomic Interactions &50 \\
Auxiliary Position Loss Coefficient &16 \\
Energy Loss Coefficient &1 \\
Optimizer &AdamW \\
Learning Rate &0.0005 \\
Learning Rate Scheduler &ReduceLROnPlateau \\
Max Gradient Norm (i.e., Clip Gradient Norm Value) &10 \\
Batch Size &1, 2, 4 \\
Training Time &100 epochs with early stopping \\
Number of Histograms &1, 8, 32, 256 \\
Number of Histogram Bins &256, 1024, 2048, 4096 \\
Histogram Bin Distribution &Uniform, Normal \\
Histogram Induced Distribution &Normal($\sigma = \mathbb{E}\left[w\right]$), Laplace($b = \mathbb{E}\left[w\right]$), Categorical, K-Categorical \\
\bottomrule
\end{tabular}
\end{table}

\subsubsection{MD17 Models}
\Cref{tab:hparams:md17:gemnet,tab:hparams:md17:schnet} show the hyperperameters for the GemNet-dT and SchNet models, respectively, as shown on \cref{tab:md17}.

\begin{table}[!htp]\centering
\caption{GemNet-dT Hyperparameters for MD17 Dataset}\label{tab:hparams:md17:gemnet}
\scriptsize
\begin{tabular}{lrr}\toprule
Hyperparameter &Value \\\midrule
Number of Radial Basis Functions &128 \\
Number of Spherical Harmonics &7 \\
Number of Building Blocks &3 \\
Number of Output Blocks &4 \\
Embedding Size of the Atoms &512 \\
Embedding Size of the Edges &512 \\
Embedding Size in the Triplet Message Passing Block &64 \\
Embedding Size of the Radial Basis Transformation &16 \\
Embedding Size of the Circular Basis Transformation &16 \\
Edge Embedding Size After the Bilinear Layer &64 \\
Cutoff Distance for Interatomic Interactions &6 \\
Number of Residual Blocks After the Concatenation &1 \\
Number of Residual Blocks in the Atom Embedding Blocks &3 \\
Number of Residual Layers Before Skip &1 \\
Number of Residual Layers After Skip &2 \\
Number of Linear Layers for the Output Blocks &3 \\
Cutoff Distance for Interatomic Interactions &6 \\
Max Number of Neighbors for Interatomic Interactions &50 \\
Auxiliary Position Loss Coefficient &16 \\
Energy Loss Coefficient &1 \\
Optimizer &AdamW \\
Learning Rate &0.001, 0.0005, 0.0001 \\
Learning Rate Scheduler &ReduceLROnPlateau \\
Max Gradient Norm (i.e., Clip Gradient Norm Value) &10 \\
Batch Size &4 \\
Training Time &100 epochs with early stopping \\
Number of Histograms X Number of Histogram Bins &128x1024, 256x128, 128x64, 2x64 \\
Histogram Bin Distribution &Uniform, Normal \\
Histogram Induced Distribution &Normal($\sigma = \mathbb{E}\left[w\right]$) \\
\bottomrule
\end{tabular}
\end{table}
\begin{table}[!htp]\centering
\caption{SchNet Hyperparameters for MD17 Dataset}\label{tab:hparams:md17:schnet}
\scriptsize
\begin{tabular}{lrr}\toprule
Hyperparameter &Value \\\midrule
Number of Hidden Channels &384 \\
Number of Filters to Use &128 \\
Number of Interaction Blocks &4 \\
Number of Gaussians &100 \\
Cutoff Distance for Interatomic Interactions &6 \\
Optimizer &Adam \\
Learning Rate &0.001, 0.0001 \\
Batch Size &4 \\
Training Time &100 epochs with early stopping \\
Number of Histograms &256 \\
Number of Histogram Bins &256 \\
Histogram Bin Distribution &Normal \\
Histogram Induced Distribution &Normal($\sigma = \mathbb{E}\left[w\right]$) \\
\bottomrule
\end{tabular}
\end{table}

\subsubsection{QM9 Models}
\Cref{tab:hparams:qm9:mxmnet} shows the hyperparameters for the MXMNet model, as shown on \cref{tab:qm9}.

\begin{table}[!htp]\centering
\caption{MXMNet Hyperparameters for QM9 Dataset}\label{tab:hparams:qm9:mxmnet}
\scriptsize
\begin{tabular}{lrr}\toprule
Hyperparameter &Value \\\midrule
Hidden Dimension Size &128 \\
Number of Hidden Layers &6 \\
Cutoff Distance for Interatomic Interactions for Global Block &5 \\
Optimizer &Adam \\
Learning Rate &0.0001 \\
Batch Size &128 \\
Training Time &100 epochs with early stopping \\
Number of Histograms &32 \\
Number of Histogram Bins &2048 \\
Histogram Bin Distribution &Normal \\
Histogram Induced Distribution &Normal($\sigma = \mathbb{E}\left[w\right]$) \\
\bottomrule
\end{tabular}
\end{table}

\clearpage
\bibliographystyleappendix{plainnat}
\bibliographyappendix{referenceappendix}

\end{document}